\documentclass{article}

\usepackage[main, final]{neurips_2026}

\usepackage[utf8]{inputenc} % allow utf-8 input
\usepackage[T1]{fontenc}    % use 8-bit T1 fonts
\usepackage{hyperref}       % hyperlinks
\usepackage{url}            % simple URL typesetting
\usepackage{booktabs}       % professional-quality tables
\usepackage{amsfonts}       % blackboard math symbols
\usepackage{nicefrac}       % compact symbols for 1/2, etc.
\usepackage{microtype}      % microtypography
\usepackage{xcolor}         % colors
\usepackage{amsmath,amssymb}
\usepackage{amsthm}
\usepackage{bbm}
\usepackage{graphicx}
\usepackage{algorithm}
\usepackage{algorithmicx}%
\usepackage{algpseudocode}%
\usepackage{mathtools}
\usepackage{enumitem}

\newtheorem{theorem}{Theorem}

\newtheorem{definition}{Definition}

\newcommand{\red}[1]{{\color{red}#1}}
\newcommand{\aud}{\texttt{aud}}
\newcommand{\model}{\texttt{mod}}
\newcommand{\Acc}{V}
\newcommand{\Stp}{S_t^{\pi}}
\newcommand{\hgamm}[1][t-1]{\hat{\gamma}_{#1}^{\model}}

\newcommand{\SPRT}{\text{LR}}
\newcommand{\UI}{\text{LR-UI}}

\newcommand{\SR}{\text{SR-}}

\title{Adaptive auditing of AI systems with anytime-valid guarantees}

\author{%
  Siyu Zhou$^{*,1}$ \quad
  Patrick Vossler$^{*,1,3}$ \quad
  Venkatesh Sivaraman$^{1,3}$ \quad
  Yifan Mai$^{2}$ \quad
  Jean Feng$^{1,3}$ \\[4pt]
  $^{1}$University of California, San Francisco \quad
  $^{2}$Stanford University \quad
  $^{3}$Weill Cancer Hub West
  \\[2pt]
  $^{*}$Equal contribution.
}

\begin{document}

\maketitle

\begin{abstract}
A major bottleneck in characterizing the failure modes of generative AI systems is the cost and time of annotation and evaluation.
Consequently, adaptive testing paradigms have gained popularity, where one opportunistically decides which cases and how many to annotate based on past results.
While this framework is highly practical, its extreme flexibility makes it difficult to draw statistically rigorous conclusions, as it violates classical assumptions: the number of observations is typically limited (often 10 to 50 cases) and decisions regarding sampling and stopping are made in the midst of data collection rather than based a pre-specified rule.
To characterize what statistical inferences can be drawn from highly adaptive audits, we introduce a hypothesis testing framework from two ``dueling'' perspectives: (i) the model's null that asserts there is no failure mode with performance below a target threshold versus (ii) the auditor's null that asserts they have a sampling strategy that will uncover a failure mode.
Leveraging Safe Anytime-Valid Inference (SAVI), we formalize the auditor as conducting ``testing by betting,'' which translates into simultaneous e-processes for testing the dueling null hypotheses.
Furthermore, \textit{if} the auditor is sufficiently powerful, we prove that these two hypotheses are asymptotically inverses of each other, in that passage of such an audit certifies the AI system as being globally robust.
Empirically, we demonstrate that our proposed testing procedures maintain anytime-valid Type-I error control, outperform pre-specified testing methods, and can reach statistically rigorous conclusions sometimes with as few as 20 observations.
\end{abstract}

\section{Introduction}

AI systems are known to exhibit various generalization failures across subgroups and input types, often described as the ``jagged frontier.''
%, making careful auditing essential for both pre-deployment safety and post-deployment improvement.
Because gold-standard evaluations for AI (agent) outputs can often be time- and/or cost-consuming \citep{Bandel2026-dr}, a common practice is to construct small, adaptively selected test suites to conduct more targeted probing of potential failure modes \citep{Husain2026-pt, Anthropic2026-ac, Yan2025-rb}.
The practical appeal of adaptive testing is that one can decide in real-time which samples to annotate and how many based on results from past annotations; approaches include checklist-style behavioral tests that probe specific linguistic capabilities \citep{Ribeiro2020-dz}; human-in-the-loop systems where the practitioner interactively refines test cases alongside model-generated candidates \citep{Ribeiro2022-zu, Gao2023-oi};  visualization interfaces to audit algorithmic fairness \citep{Cabrera2019-yo};
% subset selection that identifies the most informative items from existing benchmarks \citep{Vivek2024-ww}; 
and adversarial generation, where practitioners try to elicit worst-case behaviors via red-teaming \citep{Perez2022-qz, Ganguli2022-oo, Lee2023-vt, Li2024-qy}.
This practice even extends beyond AI, such as stress-testing devices across worst-case scenarios \citep{US-Food-and-Drug-Administration2020-lr}.

Despite its widespread use, the flexibility of adaptive testing workflows raises a fundamental statistical question: can a small number of adaptively chosen observations support rigorous inference about system robustness?
Standard statistical inference typically assumes both a pre-specified sampling scheme and sample sizes that are large enough for asymptotic theory to apply.
Adaptive testing violates both: AI practitioners often curate as few as 10 to 50 test cases \citep{Husain2026-pt, Yan2025-rb}, with the sampling and stopping strategy decided in the midst of data collection rather than based on a pre-specified rule.
Such testing regimes are traditionally disallowed as they can substantially inflate Type I error; for instance, if one continuously peeks at the p-value, the test is guaranteed to reject with probability one under the null \citep{Ramdas2025-uv}.
Even more fundamentally, it is difficult to formulate what exactly is being tested; for instance, arbitrary sampling practices cannot be used to test a model's average performance.
\textit{Thus, the goal of this work is to (i) understand what hypothesis tests can be assessed under highly flexible testing regimes and (ii) construct corresponding testing procedures that are universally valid with finite-sample guarantees.}
% Through this analysis, we also seek to address common questions of how many observations are sufficient for testing, what are the statistical limits of such procedures, and useful rules for improving efficacy of active testing procedures.

\begin{figure}
    \centering
    \vspace{-1cm}
    \includegraphics[width=\linewidth]{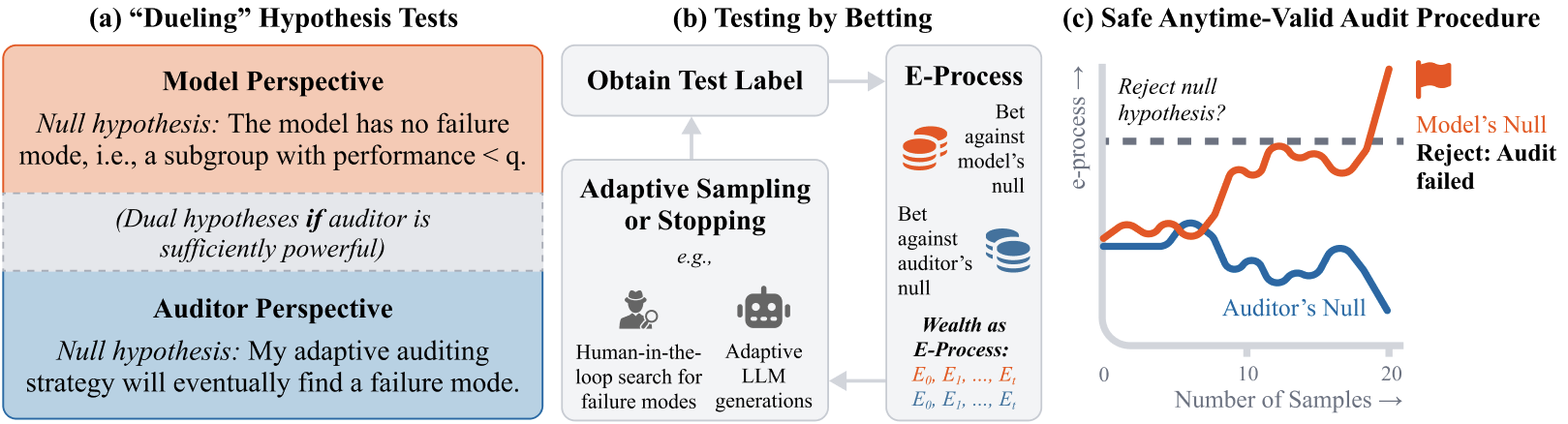}
    \caption{
    We (a) formalize flexible audits of failure modes in AI systems as dueling hypothesis tests, (b) introduce a ``testing-by-betting'' framework that is suitable for settings where sampling and stopping strategies are decided in the midst of data collection and not pre-defined, and (c) represent adaptive bets as dual e-processes with safe anytime-valid guarantees.
    % Under sufficiently powerful auditing strategies, the two hypotheses tests are asymptotically complements and the dual stopping bounds can be interpreted as certifying or rejecting AI system robustness.
    }
    \vspace{-0.5cm}
    \label{fig:main}
\end{figure}

Our proposal begins with formalizing AI robustness auditing as conducting two ``dueling'' hypothesis tests, one from the perspective of the ``model'' and one from the perspective of the ``auditor'' (Figure~\ref{fig:main}(a)).
The model's null hypothesis makes the omnibus assertion that no failure modes exist, in that every sufficiently large subgroup meets a minimum performance threshold; rejection of the model's null means that the AI model is not robust.
The auditor's null hypothesis asserts that they have a strategy that will eventually identify a failure mode; rejection of the auditor's null implies passage of this specific audit.
While these two null hypotheses were specifically constructed to be testable in practice, they are generally not exact complements.
Thus, they cannot conduct the ideal audit that yields a binary conclusion of whether an AI system is or is not robust.
But \textit{if} one has a powerful auditor that is guaranteed to find a failure mode if one truly exists, we prove that, asymptotically, the two null hypotheses are complements and thus correspond to the ideal audit.
Through this formalization, we clarify what can be tested in practice, what assumptions are additionally needed to conduct the ideal audit in theory, and how one should interpret and conduct audits.
% This theoretical result aligns with the intuition that passing a more difficult audit is a better indication of an AI system's robustness.
% Together, the theory has multiple practical implications.
% First, we can design sequential anytime-valid safe testing procedures that can \textit{simultaneously} test the two dual hypotheses.
% We introduce dual stopping bounds similar to early efficacy-futility stopping bounds commonly used in clinical trials \citep{Pocock1977-hw, OBrien1979-wr, K_K_Gordon_Lan1983-ua}; here, the procedure runs both tests simultaneously and halts the moment evidence demonstrates either discovery of a failure mode or passage of the audit (Figure~\ref{fig:dual_bound_ex}).

We then leverage Safe Anytime-Valid Inference (SAVI) methods to frame the auditor as conducting a ``Testing by Betting'' procedure, in which the auditor is allowed to make adaptive bets against any sequence of subgroups that the auditor suspects could be failure modes and where under the null, there is no winning betting strategy (Figure~\ref{fig:main} b, c) \citep{Grunwald2024-qs, Ramdas2025-uv}.
We investigate various mathematical translations of these bets into e-processes---the SAVI analog of the p-value in the sequential setting---which can be interpreted as the accumulation of wealth from the auditor's bets and, crucially, provide valid finite-sample Type-I error control under arbitrary adaptive sampling and optional stopping.
We begin with the simplest option of the likelihood ratio, then define more adaptive formulations based on universal inference, and finally introduce changepoint-based formulations.

We make the following contributions:
\vspace{-0.2cm}
\begin{enumerate}[leftmargin=*,noitemsep]
    \item We formalize AI auditing as conducting dueling hypothesis tests between the model's null and the auditor's null. We prove these are asymptotically complements under sufficiently powerful auditing strategies.
    \item We introduce dual e-processes that provide stopping bounds for simultaneous failure mode detection and audit certification, with SAVI guarantees.
    \item We provide empirical results demonstrating how adaptive testing paired with the e-process procedures can indeed provide Type-I error control and make rigorous conclusions much faster than pre-specified auditing strategies, sometimes with as few as 20 observations.
\end{enumerate}

\section{Related Work}

\textbf{LLM Evaluation and Adaptive Testing.}
Many free-form LLM outputs are expensive and time-consuming to evaluate, e.g., multiple subagent calls or need manual annotation by experts \citep{Ghosh2026-jn}.
Consequently, beyond large-scale benchmarking efforts of LLMs \citep{Liang2023-ow}, a growing body of research focuses on uncovering failure modes of LLMs through highly targeted adaptive or active testing, in which results obtained in previous rounds inform which observations to label next \citep{Ribeiro2022-zu, Gao2023-oi, Li2024-qy, Vivek2024-ww}.
These adaptive testing methods often involve human-in-the-loop, sometimes assisted by automatic systems; in the most extreme scenario, active testing is akin to red-teaming, where one constructs worst-case test scenarios \citep{Ganguli2022-oo}.
Due to their resource/cost-efficient nature, adaptive testing methods are also commonly used by AI practitioners to evaluate AI pipelines in target areas, where common recommendations are to assemble diverse test suites with 10--50 cases to cover edge cases and failure modes~\citep{Husain2026-pt, Anthropic2026-ac}.
However, there is currently no statistical framework that formalizes what exactly is and can be the target of (statistical) inference  when one conducts highly adaptive and targeted testing with unknown sampling schemes, as this violates assumptions in classical active testing settings that require a pre-specified sampling scheme~\citep{Kuang2025-uj}.
Moreover, existing works \citep{Yan2022-hv, Berrada2025-gr} often focus on active estimation of the average loss in the population of interest or in pre-specified groups instead of an omnibus testing of the existence of failure modes.
This work provides such a testing framework by allowing for arbitrary sampling and stopping schemes. %, its major contributions are both (i) precisely defining the hypothesis tests that can be tested under such settings and (ii) introducing safe any-time valid testing procedures.

\textbf{E-values and E-processes.}
Because classical p-value and hypothesis testing procedures place rigid constraints and assumptions on the data generating mechanism, recent works have proposed an alternative approach based on e-values.
E-values, whose properties are derived from (conditional) expectations rather than probability distributions, have the advantage of providing anytime-valid testing with finite-sample Type I error control under any unknown sampling schemes~\citep{Grunwald2024-qs, Ramdas2025-uv}.
Furthermore, e-values have been extended to the online setting under optional stopping, with e-processes for sequential testing and e-detectors for sequential changepoint detection \citep{Shin2024-dj, Zhou2026-zy}.
These frameworks can be viewed as testing by betting, in which the tester places a bet at each iteration and, under the null, the expected return is (no greater than) zero \citep{Shafer2021-ni, Chugg2023-ee}.
While there is growing interest in using e-values to improve generated outputs from AI systems \citep{Sadhuka2025-sf}, there are currently no e-value frameworks for highly targeted and adaptive audits of AI systems, as the exact hypothesis tests in these settings have yet to be formalized.

\textbf{Subgroup and Fairness Testing.}
The problem of auditing failure modes of AI systems falls under the broader research topic of testing and identifying subgroups with large performance disparities \citep{Chung2019-tz, Eyuboglu2023-yn, Feng2023-ua, Singh2023-ln, Subbaswamy2024-wx, Zeng2026-qt} and algorithmic fairness~\citep{Chouldechova2020-ll, Mitchell2021-wq}.
Nevertheless, the existing literature either studies the offline (batch) settings or active labeling with pre-specified sampling or subgroups  \citep{Yan2022-hv, Hartmann2026-hn}.
Our framework differs in its generality by allowing for statistically rigorous inference under active testing where neither the subgroups nor sampling strategy are pre-specified.
\vspace{-0.5cm}
% -Testing for performance disparities across subgroups is related to algorithmic fairness auditing~\citep{chouldechova2017fair, Zeng2026-qt}.
% -Recent work has developed methods for testing whether there exist subgroups with poor performance~\citep{Chung2019-tz, Eyuboglu2023-yn, Feng2023-ua, Subbaswamy2024-wx}, which is closely related to our hypothesis.
% -This framework differs in allowing fully adaptive, human-guided selection of test cases rather than algorithmic subgroup discovery; as a consequence, the methods here are actually more sample-efficient than existing methods.
% -Also related is \citep{Rauba2024-aa}, which relies on LLMs to propose multiple subgroups for testing the weakness of AI/ML pipelines in a multiple hypothesis testing setup.

\section{Methods}
\label{sec:methods}

To formalize adaptive audits of AI system robustness, the following section addresses the two main technical challenges: (i) what null hypotheses can be feasibly tested and (ii) what testing procedures provide safe anytime-valid inference?
Due to space constraints, we refer the reader to the Appendix for implementation details and detailed proofs.

\textbf{Notation.}
Let $(\mathcal{X}, \mathcal{F}, \mu)$ denote a probability space with input space $\mathcal{X}$, $\sigma$-algebra $\mathcal{F}$, and probability measure $\mu$.
For query $X \in \mathcal{X}$, let random variable $Y$ denote the score for the AI system's response (e.g., correctness), which we assume to be IID conditional on $X$.
For a measurable subset $S \in \mathcal{F}$ with prevalence $\mu(S) > 0$, the AI system's average score for subgroup $S$ is denoted $\Acc(S) \coloneqq \mathbb{E}[Y \mid X \in S]$  (we presume higher scores are better).
For fixed $\epsilon \in (0,1]$, let $\mathcal{S}_{\epsilon} \coloneqq \{S \in \mathcal{F}: \mu(S) \geq \epsilon\}$ denote all subgroups with mass at least $\epsilon$.
A \emph{failure mode} is a subgroup $S \in \mathcal{S}_{\epsilon}$ on which $\Acc(S)$ falls below a target threshold $q$.
For any positive integer $r$, $[r]$ corresponds to the sequence $\{1, 2, \dots, r\}$.
% Let $X_t$ and $Y_t$ denote the query and the score at time $t \in \mathbb{N}$ respectively. Define $(\mathcal{F}_t)_{t\in\mathbb{N}} $ as the natural filtration with $\mathcal{F}_t = \sigma\left(X_1, Y_1, \dots, X_{t-1}, Y_{t-1}\right)$.

\subsection{Dueling hypothesis tests for auditing AI robustness}

Defining a suitable hypothesis testing framework requires overcoming a fundamental asymmetry in adaptive model auditing: while identifying a failure mode proves a lack of robustness, passing a targeted test suite does not certify its robustness.
This is true from both the software engineering perspective of unit tests as well as the Fisherian perspective of statistical hypothesis testing \citep{Christensen2005-bp}.
Furthermore, because adaptive sampling opportunistically seeks out difficult cases, it cannot make conclusions about typical estimands like average or subgroup-specific performance.

We resolve this by framing AI robustness audits as conducting two ``dueling'' hypothesis tests:
the \emph{model's null} makes the omnibus assertion that no failure modes exist, while the \emph{auditor's null} asserts that their specific auditing strategy will eventually find a failure mode.
We then investigate conditions under which testing these two null hypotheses corresponds to the ideal audit, where passage of the audit occurs if and only if the AI system is globally robust.

\subsubsection{The Model's hypothesis}

The model's null hypothesis states that it performs well across all subgroups with prevalence at least $\epsilon$, while the alternative states that there exists one or more failure modes.
This can be mathematically formalized as
\begin{align}
\begin{split}
H_0^{\model}: & \quad \forall S \in \mathcal{S}_{\epsilon}, \quad \Acc(S) \geq q,\\
H_1^{\model}: & \quad \exists S \in \mathcal{S}_{\epsilon} \text{ such that } \Acc(S) < q.
\label{eq:mod_hypo}
\end{split}
\end{align}
We emphasize that this is a subgroup-uniform testing problem rather than an estimation problem: the procedure does not need to recover the worst-case subgroup nor its accuracy. 
% Instead, it accumulates evidence along the adaptively chosen sequence $S_t^\pi$; under $H_0^{\model}$, every admissible selected subgroup has conditional mean at least $q$, which is the property used by the e-processes below.

The closest analog to this test appears in the \textit{multicalibration} literature \citep{Hebert-Johnson2018-vb,Feng2023-ua}, which studies whether a model is well-calibrated over a collection of subgroups.
Similar questions are also considered in the Distributionally Robust Optimization (DRO) literature where one may be interested in the worst-case subgroup performance \citep{Subbaswamy2021-np}.
The key difference in this work is that we study adaptive anytime testing, whereas prior works focused on the batch setting with pre-specified sampling schemes.

\subsubsection{The Auditor's hypothesis}

The auditor's null hypothesis asserts that they have an unspecified, adaptive, and potentially adversarial auditing strategy that will eventually uncover a failure mode.
% , that will discover a failure mode within some budget $m$,
Formally, auditing strategy $\pi$ is defined as a function of the natural filtration at time $t-1$, denoted $\mathcal{F}_{t-1}$, to the adaptively selected subgroup at time $t$, denoted $\Stp \in \mathcal{S}_\epsilon$.
At each time $t$, the auditor bets against a subgroup $\Stp$, samples an observation $X_t$ from it (with respect to $\mu$ restricted to $\Stp$), and then obtains the corresponding score $Y_t$.
Formally, we represent the auditor's hypothesis test as
\begin{align}
\begin{split}
    H_0^{\aud, [m]}: &  \ \exists \tau \in [m], \ s.t. \ \forall t \geq \tau, \ \Acc(\Stp) < q,\\
    % H_1^{\aud,[m]}: & \not \exists \tau \in [m], \ s.t. \ \forall t \geq \tau, \ \Acc(\Stp) < q.
    H_1^{\aud,[m]}: & \ \forall \tau \in [m], \ \exists t \geq \tau, \ s.t.  \ \Acc(\Stp) \geq q.
\label{eq:aud_hypo}
\end{split}
\end{align}
Note that \eqref{eq:aud_hypo} is defined with respect to a hyperparameter $m$, which states that the auditor has a budget of size $m$ to find the failure mode.
This budget is necessary to ensure the null hypothesis is testable in practice (otherwise one would have to wait forever).
Nevertheless, neither data collection nor the auditor's adaptive sampling strategy need to be frozen at time $m$, nor should it be.
Instead, as discussed later, budget $m$ is simply the time at which we may begin seriously testing the auditor's null hypothesis.

\subsubsection{Duality of the model's and auditor's null}
\label{sec:duality}

In practice, one can only test the model's and auditor's null hypotheses.
Nevertheless, because the auditor's null only makes a statement about a specific audit strategy, it is natural to ask what conditions are needed to certify that an AI system is \textit{globally} robust.
We show that this is possible if the auditor's strategy is \emph{asymptotically consistent}, as defined below.
\begin{definition}[Asymptotic consistency]
    The auditor's strategy $\pi$ is asymptotically consistent if the performance of the selected subgroups converges to below $q$, i.e., $\lim_{t \rightarrow \infty} \Acc(S^{\pi}_t) < q$, whenever there exists a subgroup $S^* \in \mathcal{S}_{\epsilon}$ with performance below $q$.
    % If no such subgroup exists, the strategy only outputs subsets where $\Acc(S^{\pi}_t) \geq q$ for all $t$.
\end{definition}

Many auditing strategies are asymptotically consistent.
The simplest option is to conduct IID sampling and estimate $V(x) = \mathbb{E}(Y|X = x)$ with some model $\hat{V}_n$.
Standard results regarding the convergence of such estimators imply this asymptotic consistency condition \citep{Anthony1999-iv, Biau2021-io, Scornet2015-ma}.
A more efficient option is to have the auditor use active learning methods, which sample observations to maximize learning and provide similar consistency results (see e.g., \cite{Goetz2018-bq} and \cite{Zhu2022-hp}).
In fact, the bandit literature may offer solutions that best match the problem faced in adaptive auditing: the goal of finding the worst performing subgroup is equivalent to the continuous multivariate bandit with arms corresponding to $X$ and the reward corresponding to the negative score $-Y$ \citep{Auer2002-rl, Srinivas2009-me, Zhou2020-hw}.
Regret bounds from such algorithms imply the asymptotic consistency condition, which inspires the use of an Upper Confidence Bound (UCB)-based auditor in Section~\ref{sec:experiments}.

Conditional on having a consistent auditing strategy, we have the following duality result that states that the model's and auditor's null hypotheses are asymptotically complements.% (Fig~\ref{fig:duality}).
\begin{theorem}
\label{thm:duality}
Conditional on an asymptotically consistent auditing strategy $\pi$, the model's and auditor's hypotheses are complementary asymptotically as the budget $m$ increases, i.e.,
    \begin{align*}
        \lim_{m \rightarrow \infty} H_1^{\aud, [m]} \text{ true} \iff H_0^{\model}\text{ true} \quad \text{and} \quad
        \lim_{m \rightarrow \infty} H_0^{\aud, [m]} \text{ true} \iff H_1^{\model}\text{ true}.
    \end{align*}
    % \begin{align*}
    %     \mu\left( H_1^{\aud} \text{ is true}\right) = & \mu\left( H_0^{\model} \text{ is true}\right) \\
    %     \mu\left( H_0^{\aud} \text{ is true}\right) = & \mu\left( H_1^{\model} \text{ is true}\right).
    % \end{align*}
\end{theorem}
% Most standard nonparametric AI/ML models are known to be asymptotically consistent, e.g., neural networks, splines, tree-based models, etc \citep{Anthony1999-iv, Cox1984-db, Scornet2015-ma}.
Theorem~\ref{thm:duality} implies that if the auditor is strong enough, testing the two ``dueling'' hypotheses with a sufficiently large budget $m$ can in fact be interpreted as providing \textit{dual} stopping bounds for certifying or rejecting an AI system as being globally robust.
Such dual stopping bounds can also be viewed as a generalization of those used in clinical trials to declare efficacy versus futility of a treatment \citep{Pocock1977-hw, OBrien1979-wr}.
The key difference is that we focus on an omnibus test with respect to all subgroups $\mathcal{S}_{\epsilon}$, whereas classical dual stopping bounds are solely concerned with the average treatment effect with respect to the full population.

Practically speaking, this result suggests leveraging both expert- and data-driven methods when auditing AI systems, as the former can provide useful prior knowledge in settings with limited data while the latter can ensure asymptotic consistency.
% This result formalizes the importance of a strong auditor when making AI system robustness claims.
At the same time, %\red{in a finite sample setting, rejecting the auditor’s null means sufficient evidence that the model passes this specific auditor at this specific budget, and thus 
this result warns against making global robustness claims when the auditor is weak, echoing recent concerns that poorly designed red-teaming exercises may be akin to ``security theater'' \citep{Feffer2024-sp}.

%This duality of the model's and auditor's hypotheses make it possible to interpret the rejection of the auditor's null as a robustness certification of the AI system rather than a limited statement regarding the particular auditing strategy implemented.

\subsection{Dual safe anytime-valid testing procedures}
\label{sec:testing_procedure}

We now investigate procedures that translate fully adaptive sampling schemes and optional stopping into statistics for testing the model's and auditor's hypotheses while providing finite-sample guarantees.
Motivated by the ``Testing by Betting'' framework \citep{Shafer2021-ni}, we formulate testing procedures as e-processes \citep{Ramdas2025-uv}, a key mathematical object from the SAVI literature:
\begin{definition}[E-process]
A sequence $(E_t)_{t \geq 0}$ adapted to a filtration $(\mathcal{F}_t)$ is an \emph{e-process} with respect to $H_0$ if $\mathbb{E}_P[E_\tau] \leq 1$ for any stopping time $\tau$ and any $P \in H_0$. 
Equivalently, $(E_t)_{t\geq 0}$ is an e-process if there exists a test supermartingale family $(M_t^P)_{P \in H_0}$ such that, $\forall P \in H_0$,  (i) $M_t^P \geq 0$ $P$-almost surely, (ii) $M_t^P$ is a supermartingale under $P$, and (iii) $\mathbb{E}_P[M_0^P] \leq 1$, and that $E_t \leq M_t^P$ $P$-almost surely.
\end{definition}
From a testing perspective, the key property of e-processes is that they are guaranteed to control the Type I error rate under adaptive sampling and optional stopping by Ville's inequality.
That is, for any $P \in H_0$ and $\mathcal{F}$-stopping rule $\tau$, an e-process stopped at threshold $1/\alpha$ has a Type I error rate no more than $\alpha$, i.e.,
\begin{equation}
P\left(E_\tau \geq \frac{1}{\alpha}\right) \leq \alpha. \label{eq:Ville}
\end{equation}
As such, we next present various e-processes for testing the model's and auditor's hypotheses, beginning with the simplest option of the likelihood ratio and building up to more complex constructions.
For ease of exposition, this section focuses on binary scores $Y$ (e.g., accuracy) but the results extend to continuous or multiclass scoring.

\subsubsection{Testing the model's null}

To test the model's null hypothesis, we consider the likelihood ratio process, an adaptive version based on Universal Inference (UI) \citep{Wasserman2020-mn}, and extensions to both procedures by additionally integrating changepoint detection \citep{Shiryaev1963-xq, Roberts1966-yr, Shin2024-dj}. 

\textbf{Option 1: Likelihood ratio process (\texttt{LR}).}
One of the most well-known approaches to construct an e-process is to use likelihood ratios, which is based on the classical  Sequential Probability Ratio Test (SPRT) \citep{Wald1945-pi}.
In particular, let $p(Y;\gamma)$ be the likelihood of a Bernoulli distribution with probability $\gamma$, and define the \texttt{LR} test statistic as $E_t^{\SPRT, \model}:= E_{1:t}^{\SPRT, \model}$ where
\begin{align}
    E_{j:t}^{\SPRT, \model} = \prod_{k=j}^t \frac{p(Y_k; q-\delta)}{p(Y_k; q)}. \label{eq:M_LR_model}
\end{align}
% Can we just frame delta as a separation/power parameter for this fixed LR process. Near boundary alternatives remain failures, but this fixed LR is not designed to guarantee fast detection there.
Here, $\delta > 0$ is a fixed user-specified separation parameter.
When the auditor samples subgroups with $\Acc(S_t^\pi) \leq q-\delta$, the process has positive expected log-growth against the boundary value $q$. 
Subgroups with $\Acc(S_t^\pi) \in (q-\delta,q)$ remain failures under $H_1^{\model}$ but may produce slow growth, if any.
Under the null $H_0^{\model}$, the test statistic $E_t^{\SPRT, \model}$ is an e-process, with no expected growth.
Therefore rejecting $H_0^{\model} $ at time $t$ if $E_t^{\SPRT, \model} \geq \frac{1}{\alpha}$ is an anytime-valid level-$\alpha$ sequential test.

\textbf{Option 2: Likelihood Ratios with UI (\texttt{LR-UI}).}
While simple to construct, the disadvantage of the \texttt{LR} procedure above is that the fixed probability $q -\delta$ in \eqref{eq:M_LR_model} can be suboptimal for detecting a particular alternative in $H_1^{\model}$. 
% In fact, LR is powerful when the pair $(q, q-\delta)$ is least favorable in the sense that it is hardest to control Type I and Type II error when the $V(S)$ is $q$ and $q-\delta$ respectively, an example of which is to test
% \begin{align*}
%     H_0:&\ \forall \ S \in \mathcal{F} \text{ with } \mu(S) \geq \epsilon, \  \Acc(S) \geq q,  \\
%     H_1:&\ \forall \ S \in \mathcal{F} \text{ with } \mu(S) \geq \epsilon ,\  \Acc(S) \leq q - \delta.
% \end{align*}
% However, $H_1^{\model}$ only asserts the existence of $S$ with $\mu(\{S \in \mathcal{F}: \Acc(S) < q\}) \geq \epsilon$. 
% The Type II error of LR could be high when $\mu(\{S \in \mathcal{F}: \Acc(S) \leq q-\delta\})$ is close to 0.
To improve its power, we introduce an adaptive version denoted \texttt{LR-UI} that draws on UI to define the test statistic $\quad E_t^{\UI, \model} := E_{1:t}^{\UI, \model}$, where
\begin{align}
    E_{j:t}^{\UI, \model} = \prod_{k = j}^t\frac{p\left(Y_k; \hat{\gamma}_{k-1}^{\model}\right) }{p\left(Y_k; q\right)}. \label{eq:M_LR_a_model}
\end{align}
Here $\hat{\gamma}_{t-1}^{\model}$ is any estimator of $V(X_t)$ with respect to the natural filtration $\mathcal{F}_{t-1}$ that outputs a probability consistent with $H_1^{\model}$.
By using an estimator, one can view the auditor as placing an adaptive bet on not only which subgroup is a potential failure mode, but also its likely failure rate, in the hopes to further optimize its gain in the testing-by-betting framework.
Because $Y_t$ and $\hat{\gamma}_{t-1}^{\model}$ are conditionally independent given $\mathcal{F}_{t-1}$, $E_t^{\UI, \model}$ is a test supermartingale and thus an e-process for $H_0^{\model} $.

\textbf{Options 3 and 4: SR Likelihood ratio processes (\texttt{SR-LR}, \texttt{SR-LR-UI}).}
\texttt{LR} and \texttt{LR-UI} are optimal when the auditor identifies failure modes from the very beginning, but an auditor may require multiple attempts (e.g., a learning period) in practice.
Consequently, we consider integrating changepoint detection techniques into the aforementioned e-processes, so that the testing procedures are powered to detect the emergence of failure modes in the adaptive testing sequence.

In particular, we will consider an extension based on the Shiryaev-Roberts (SR) procedure \citep{Shiryaev1963-xq, Roberts1966-yr}, which is classically defined as the sum of LR processes starting from all possible candidate changepoints and has been more recently studied under the umbrella of e-detectors \citep{Shin2024-dj}.
Because the weighted sum of e-processes is also an e-process, we define the test statistics \texttt{SR-LR} and \texttt{SR-LR-UI} as the e-processes
\begin{align*}
    E_t^{\SR\SPRT, \model} = \sum_{j=1}^{t} w_j E_{j:t}^{\SPRT, \model}, \quad
    % E_t^{\SR\Mix, \model} = \sum_{j=1}^{\min(t, m)} w_j E_{j:t}^{\Mix, \model}, \\
    E_t^{\SR\UI, \model} = \sum_{j=1}^{t} w_j E_{j:t}^{\UI, \model}.
\end{align*}
for pre-defined weights $\{w_j : j=1,2,\cdots\}$ that sum to one.
In the betting framework, the SR extension equates to the auditor distributing their wagers across all candidate changepoints to hedge against the unknown delay before a failure mode emerges.

\subsubsection{Testing the auditor's null}

To test the auditor's null hypothesis, we cannot directly apply the aforementioned solutions because the auditor's null has a complex composite structure whereas the model's null is much simpler.
% This challenge is most apparent from the changepoint interpretation of the auditor's null, which can be viewed as asserting there is a drop in the subgroup scores, i.e., a changepoint, within the auditor's allotted budget $m$; the alternative states that no such changepoint exists.
% Rejecting this null hypothesis prior to time $m$ is risky and an over-extrapolation into the future behavior of the auditor, which cannot be anticipated since it is fully adaptive.
The challenge is most apparent when viewing the auditor’s null as a changepoint statement. 
In this formulation, the null asserts the presence of a drop in the subgroup scores - a changepoint - \emph{within} the allotted budget $m$, whereas the alternative asserts the absence of such a drop.
Rejecting the null prior to observing \emph{all} $m$ observations is therefore risky, since doing so entails an unwarranted extrapolation about the future of an auditor adaptive to the \emph{full} past trajectory.

As such, one may ask if it is even possible to meaningfully test the auditor's null prior to time $m$.
We prove that this is indeed impossible, in that there is no anytime-valid testing procedure that can achieve higher power than its Type I error prior to the $m$-th observation.
\begin{theorem}
    Let $\phi(t)$ be any anytime-valid testing procedure for the auditor's null hypotheses such that the Type I error rate is controlled at $\alpha$, 
    i.e., $\sup_{m' \in [m], t < m'} \mathbb{E}\left (\phi(t)|H_0^{\aud,m'} \text{ is true} \right) \leq \alpha $ where $H_0^{\aud,m'}$ is the singleton analog of the auditor's null hypothesis. 
    Then, the power prior to time $m$ is also no more than $\alpha$, i.e., $\sup_{m' \in [m], t < m'} \mathbb{E}\left(\phi(t) | H_1^{\aud,m'} \text{ is true}\right) \leq \alpha$.
    \label{thm:auditor_start}
\end{theorem}

% In addition, the conditional Ville's inequality suggests that any any-time valid test remains to be any-time valid conditional on a number of observations at the beginning of the data stream.
Given this result, we can simplify the testing procedure to wait until the end of the budget window $m$.
Then starting at time $m$, we can test the simplified null hypothesis
\begin{align}
\begin{split}
    H_0^{\aud, m}:& \ \forall t \geq m, \ \Acc(\Stp) < q, \\
    H_1^{\aud, m}:& \ \exists t \geq m, \ s.t. \ \Acc(\Stp) \geq q,
    \label{eq:simple_aud}
\end{split}
\end{align}
for which the testing procedures from the previous section are now applicable.

\textbf{Option 1. Likelihood Ratio Process (\texttt{LR}).}
The first option is to define a likelihood ratio process similar to \eqref{eq:M_LR_model}, but starting from time $m$ and with the larger probability parameter in the numerator.
The \texttt{LR} test statistic for the auditor's null is thus
\begin{align}
    E_{t}^{\SPRT, \aud} = 
    \begin{cases}
        \prod_{k = m}^t \frac{p(Y_k; q+\delta')}{p(Y_k; q)}, & \text{ if } t \geq m, \\
        1, & \text{ otherwise,}
    \end{cases}
    \label{eq:M_LR_auditor}
\end{align}
for some fixed tolerance $\delta' > 0$.

% For the auditor-side fixed LR, using $q-\delta$ in the denominator tests a separated version of the auditor's null in which, after the budget $m$, the selected subgroups remain clear failures. 
% For a user-specified separation parameter $\delta > 0$, the approximate auditor's hypothesis is
% \begin{align}
%     H_0^{\aud, m, \delta}:& \ \forall t > m, \ \Acc(\Stp) \leq q - \delta, \\
%     H_1^{\aud, m, \delta}:& \ \exists t > m, \ s.t. \ \Acc(\Stp) \geq q.
% \end{align}
% The interval $(q-\delta,q)$ is therefore an indifference zone for this separated certification test: these accuracies are still below the target $q$, but the separated LR does not provide Type~I error control for the full auditor null on that band.

% Define
% \begin{align*}
%     M_{(m+1):t}^{\SPRT, \aud} = 
%     \begin{cases}
%         \prod_{k = (m+1)}^t\frac{p\left(Y_k; q\right) }{p\left(Y_k; q - \delta\right)}, & \text{ if } t > m, \\
%         1, & \text{ otherwise. }
%     \end{cases}
% \end{align*}
% Rejecting $H_0^{\aud, m, \delta}$ when $M_{(m+1):t}^{\SPRT, \aud} > \frac{1}{\alpha}$ is an any-time valid level-$\alpha$ test.

\textbf{Option 2. Likelihood Ratios with UI (\texttt{LR-UI}).}
Alternatively, we can again use an adaptive version of the likelihood ratio process to avoid specifying a fixed alternative.
That is, the \texttt{LR-UI} test statistic for the auditor's null is
\begin{align}
    E_{t}^{\UI, \aud} = 
    \begin{cases}
        \prod_{k = m}^t\frac{p\left(Y_k; \hat{\gamma}_{k-1}^{\aud}\right) }{p\left(Y_k; q \right)} , & \text{ if } t \geq m, \\
        1, & \text{ otherwise,}
    \end{cases}
\end{align}
%where $\hat{\gamma}_{t-1}^{\aud}$ is an estimator of $\Acc(X_t)$ with respect to $\mathcal{F}_{t-1}$ under the auditor's alternative hypothesis. 
where $\hat{\gamma}_{t-1}^{\aud}$ is an estimator of $\Acc(X_t)$ with respect to the natural filtration $\mathcal{F}_{t-1}$ under $H_1^{\aud, m}$.

\subsubsection{Dual testing}
% 1. Procedure overview: combine two e-processes, single auditor selecting from $\mathcal{S}_\epsilon$, stop at first crossing, point to Alg 1?
% 2. FWER Theorem
% 3. Choices for parameters?
% 4. talk about mutual exclusivity?

Combining the e-processes for the model's and auditor's nulls yields dual safe anytime-valid procedures for any adaptive auditor (Figure~\ref{fig:main}; Algorithm~\ref{algo:dual_testing}).
At each round $t$, the model's test statistic is updated; once $t \geq m$, the auditor's test statistic is monitored alongside it.
% The nonadaptive auditor's test $E_t^{\SPRT, \aud}$ is paired with one nonadaptive model's test $E_t^{\SPRT, \model}$ or $E_t^{\SR\SPRT, \model}$ and so is the adaptive variant.
%When the model's test is adaptive (resp. nonadaptive), the auditor is also adaptive (resp. nonadaptive).
The procedure stops as soon as either statistic exceeds $1/\alpha$.
% Each bound is individually calibrated at level $\alpha$ by Ville's inequality on its e-process.
% A direct application of a Bonferroni correction would split this level, setting each bound at $\alpha/2$ to keep family-wise error at $\alpha$.
% This correction is unnecessary because the auditor's strategy $\pi: \mathcal{F}_{t} \to \mathcal{S}_{\epsilon}$ makes the two null hypotheses mutually exclusive:  under $H_0^{\model}$, every selected subgroup has $\Acc(\Stp) \ge q$, contradicting $H_0^{\aud, m}$.
Because at most one of the two nulls is true under any data distribution, Bonferroni correction is unnecessary and the family-wise error rate (FWER) is equal to the Type I error of the individual tests, as formalized below.
\begin{theorem}
\label{thm:fwer}
    For any auditor strategy $\pi$, the dual testing procedure with stopping thresholds $1/\alpha$ provides safe anytime-valid control of the FWER at level $\alpha$. 
    % $P\big(\text{procedure rejects any true null in } {H_0^{\model}, H_0^{\aud, m}}\big) \le \alpha,$
    % over data-generating distributions $P$.
\end{theorem}
The main hyperparameter of the dual testing procedure is the choice of the auditor's budget $m$, where a higher value means a more stringent audit is conducted.
This should be selected based on prior knowledge and real-world constraints.
One should generally choose a higher $m$ to mitigate failure modes in high-risk settings or find failure modes with lower prevalence.
If the goal is to ensure a typical user is unlikely to encounter a failure mode, $m$ can also be chosen to approximate how many test cases a typical user may conduct to assess the robustness of an AI system.
That said, our ablation study (in the Appendix) shows that the procedure is relatively robust to the choice of $m$.
More detailed guidance for hyperparameter tuning can be found in the Appendix.

\section{Experiments}
\label{sec:experiments}

We evaluate the dual anytime-valid testing framework in two experimental settings.
Experiment~1 uses semi-synthetic data to comprehensively test a variety of settings that vary the size and magnitude of the failure mode.
Experiment~2 evaluates the practical utility of the methods on a real-world LLM pipeline for clinical note analysis.
Failure modes are defined as poor-performing subgroups with prevalence $\epsilon$ at least $0.05$.
All hypothesis tests were conducted at level $\alpha = 0.05$.

We test several auditing strategies and testing procedures, with the primary goal of comparing the power of the testing procedures given the same auditing strategy.
First, we apply the simplest \texttt{LR} e-process formulation with three different fixed auditing strategies: \texttt{Stratified} (uniform sampling across known subgroups), \texttt{Pre-learned} (pre-learn a model for sampling, based on the batch approach as in \citep{Feng2023-ua}), and \texttt{Oracle} (perfect subgroup knowledge in Experiment 1, which upper bounds performance) or \texttt{Informed Prior} (informed sampling of a known failure mode in Experiment 2, though not necessarily optimal).
% Nevertheless, we emphasize that the primary goal of the experiments is to evaluate the power of the testing procedures holding the sampling process constant, not the power between different sampling processes.
Next, we test the four e-processes from Section~\ref{sec:testing_procedure} (\texttt{LR}, \texttt{LR-UI}, \texttt{SR-LR}, and \texttt{SR-LR-UI}) with an adaptive auditing strategy based on online learning.
% This adaptive auditor leverages an online learning strategy to improve its prediction of scores $Y$.
This adaptive auditor converts LLM queries ($X_t$) into embeddings and retrains an NN ensemble with each additional annotation.
Based on the predicted scores, the auditor samples from the $\epsilon$-sized subgroup with the lowest predicted scores below the threshold $q$, if such a subgroup exists; otherwise, based on the Upper Confidence Bound (UCB) algorithm \citep{Auer2002-rl}, it samples the bottom half of the population with the lowest confidence bounds for the estimated performance.
The adaptive alternative probabilities in \texttt{LR-UI} and \texttt{SR-LR-UI} are constructed using the Exponentially Weighted Averaging Forecaster (EWAF) algorithm \citep{Cesa-Bianchi2006-tl}.
Note that for all experiments, the auditor does not have direct access to strata information, unless specified otherwise; in doing so, our goal is to mimic real-world settings where strata information is often unavailable.

In the Appendix, we provide full experimental details, demonstrate that the test procedures are robust to the choice of budget $m$ in ablation studies, and show results with similar trends on the CUB-200-2011 image dataset \citep{Wah2011-kc}.

\subsection{Experiment 1: Detecting failure modes with semi-synthetic data}

To comprehensively evaluate testing methods across a variety of conditions, we simulated an AI pipeline with varying degrees of degradations in the GPQA dataset \citep{Rein2023-dv}.
Following the available strata of science domains (Biology, Chemistry, Physics) and education levels (Undergraduate, Graduate, Post-graduate), we simulated a globally robust model with no failure modes (\emph{None}) and models with a failure mode in the Chemistry domain, particularly for higher education levels (\emph{Small}, \emph{Medium} and \emph{Large degradation}).
% Difficulty of detection increases across \emph{Easy}, \emph{Medium} and \emph{Hard}.
% To verify the Type I error of the model's null and power for rejecting the auditor's null, we also include a \emph{model's null-true} setting.
We conduct experiments across two settings commonly found in practice: (a) the auditor has access to an unlabeled dataset from which it can actively choose observations to label, and (b) the auditor has no pre-existing data and must actively generate test cases. 
We simulate the latter by generating data from an LLM, and then use the auditor's accuracy model to prioritize which observations to annotate; see Appendix~\ref{sec:active_learning} for a more detailed discussion on the interpretation of such an auditor.
% knows categorical attributes (domain and difficulty level) and must request questions to be generated from a $3 \times 3$ grid of categories, then selects among generated candidates using embeddings.
We evaluate each method over 100 simulation replicates, measuring rejection rates of $H_0^{\model}$ (power to detect failure modes) and $H_0^{\aud}$ (audit passage), as well as the number of observations required to reach a decision.

\begin{figure}
    % \vspace{-0.8cm}
    \centering
    \begin{minipage}{.5\textwidth}
\centering
  {\small \textit{(a) Auditor selects which unlabeled cases to annotate}} \\
    \begin{minipage}{.495\textwidth}\includegraphics[width=\textwidth]{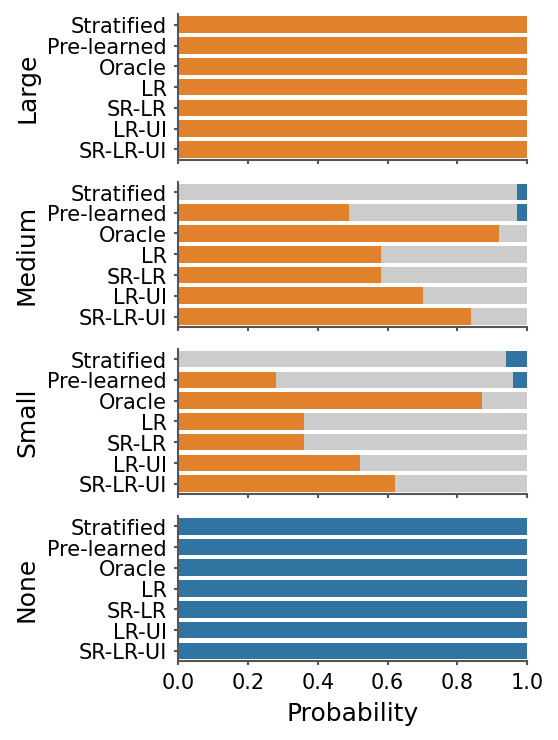}\end{minipage}%\hspace{8pt}%
    \begin{minipage}{.495\textwidth}\includegraphics[width=\textwidth]{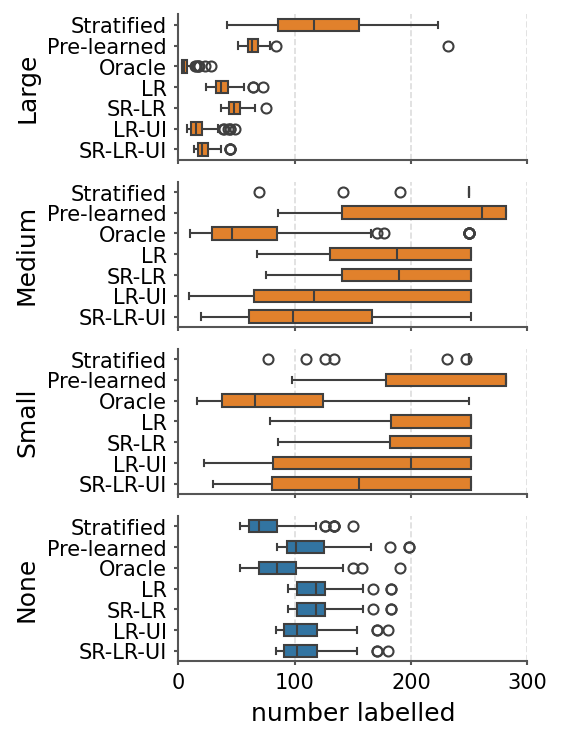}\end{minipage}
\end{minipage}%`
\begin{minipage}{.5\textwidth}
\centering
  {\small \textit{(b) Auditor generates test cases to annotate}}\\
  \begin{minipage}{0.495\textwidth}\includegraphics[width=\textwidth]{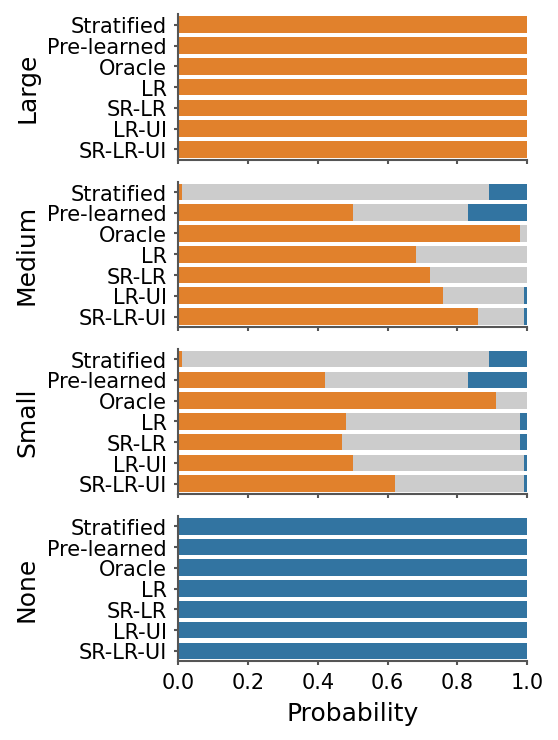}
    \end{minipage}%\hspace{8pt}%
    \begin{minipage}{0.495\textwidth}\includegraphics[width=\textwidth]{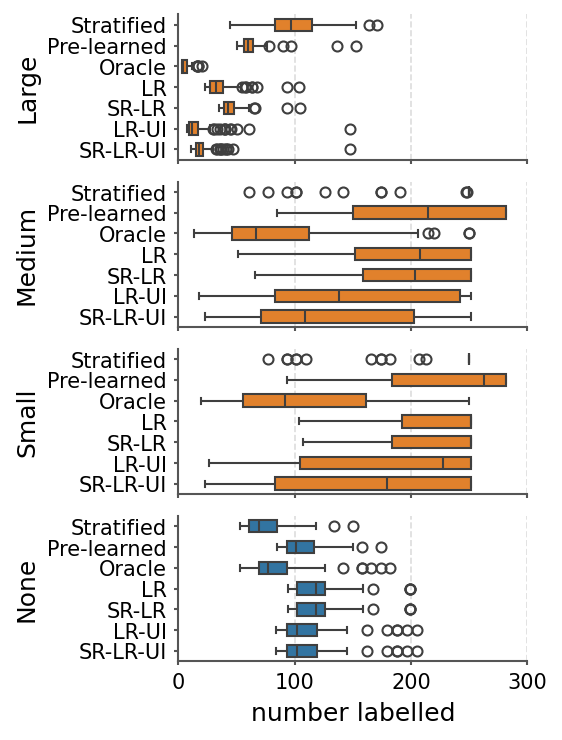 }\end{minipage}
\end{minipage}
\vspace{-0.1cm}
\begin{minipage}{0.25\textheight}
    \centering
    \includegraphics[width = \textwidth]{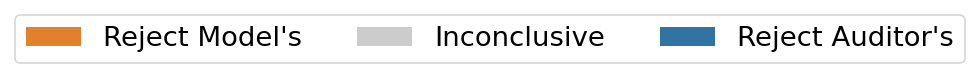}
\end{minipage}
    % \vspace{-0.2cm}
    % \textit{(a) Auditor selects which unlabeled cases to annotate}
    % \includegraphics[width=\linewidth]{images/experiment1a_power.png}\\
    % \includegraphics[width=\linewidth]{images/experiment1a_arl.png}
    % \textit{(b) Auditor generates test cases to annotate}
    % \includegraphics[width=\linewidth]{images/experiment1b_power.png}
    % \includegraphics[width=\linewidth]{images/experiment1b_arl.png}
    % \vspace{-0.5cm}
    \caption{Detecting failure modes with semi-synthetic data (Experiment~1). 
    The rows correspond to settings where the model has a failure mode with decreasing magnitude/prevalence of the degradation (\emph{Large}, \emph{Medium}, \emph{Small}), where the bottom corresponds to \emph{No failure mode}.
    % Difficulty of failure mode detection increases from left to right: Easy, Medium, Hard and Model's Null True.
    Left column: rate of rejecting dueling null hypotheses, where orange indicates rate of detecting a failure mode, blue indicates passing the audit, and gray corresponds to neither (inconclusive).
    % In the \emph{Easy}, \emph{Medium} and \emph{Hard} settings, orange bars show power of detecting a failure mode and green indicate passing the audit.
    % In the \emph{No failure mode} setting, blue bars show falsely detecting a failure mode
    % and blue power of passing the audit. 
    % In all settings, green bars show the rate of inconclusiveness at the end of the audit.
    %\texttt{SR-LR-UI} has the highest power of rejecting model's null other than \texttt{Oracle}.
    % Type I error is well controlled for all methods.
    %(power for $H_0^{\model}$, Type~I error for $H_0^{\aud}$); 
    Right column: boxplots of the number of cases annotated to reach the correct conclusion or the maximum sample size.
    %\red{comment on the speed when Oracle-oracle is ready}
    % \texttt{LR}, \texttt{SR-LR}, \texttt{LR-UI} and \texttt{SR-LR-UI} are the e-processes from Section~\ref{sec:testing_procedure} applied with active sampling. 
    % These four labels correspond to the model's test statistic implemented and the auditor's test statistic is matched based on adaptivity.
    % \texttt{Stratified}, \texttt{Pre-learned} and \texttt{Oracle} are \texttt{LR} e-processes with corresponding passive sampling strategies. 
    % ; SR-LR-UI with active sampling achieves the best practical performance. Stratified sampling frequently fails to detect existing failure modes, instead rejecting the auditor's null.
    }
    \label{fig:exp1}
\end{figure}

Among the adaptive testing methods, \texttt{SR-LR-UI} achieved the highest power for detecting failure modes (Fig~\ref{fig:exp1}).
UI variants (\texttt{SR-LR-UI}, \texttt{LR-UI}) generally outperformed their fixed-probability counterparts (\texttt{SR-LR}, \texttt{LR}), as using a probability estimator can further optimize the auditor's bets and thus power for detecting failure modes.
Integrating changepoint detection also improves power for UI procedures (i.e., \texttt{SR-LR-UI} $>$ \texttt{LR-UI}), though it had a minimal effect for  fixed-probability methods.
The reason is that in the setting where the auditor can adaptively place bets against the failure rate, \texttt{SR-LR-UI} can tune the alternative probabilities to match the failure mode identified by the auditor after some initial learning period, while \texttt{LR-UI} is constrained to finding alternative probabilities that work well from the very beginning.
% \texttt{SR-LR-UI} (resp. \texttt{SR-LR}) is a weighted average of \texttt{LR-UI} (resp. \texttt{LR}) sub-processes, each assuming one possible changepoint, while \texttt{LR-UI} (resp. \texttt{LR}) implicitly assume a changepoint at the start and can be seen as a special case of \texttt{SR-LR-UI} with $w_1 = 1$ and $w_t = 0$ for $t = 2, 3, \dots$. 
% Both \texttt{SR-LR-UI} and \texttt{LR-UI} need to cumulate evidence against $H_0^{\model}$ from observations post true changepoint to overcome noise prior changepoint.
% With adaptively estimated probabilities, \texttt{SR-LR-UI}'s sub-processes post the true changepoint can easily dominate those prior and thus cumulate more evidence against $H_0^{\model}$ than \texttt{LR-UI}, leading to improved power of \texttt{SR-LR-UI} over \texttt{LR-UI}.
In contrast, in the fixed-probability setting, the two methods are more similar because the alternative probabilities are fixed to be the same, leaving little room for optimizing power.

As expected, adaptive testing strategies substantially outperformed the pre-specified auditing approaches of \texttt{Stratified} and \texttt{Pre-learned} sampling.
% The adaptive testing strategies substantially outperform the more passive approaches of \texttt{Pre-Learned}, which performs poorly if there is insufficient data prior to testing for  reliably identifying failure modes, and \texttt{Stratified}, which spreads observations uniformly across subgroups rather than concentrating on problematic regions.
Furthermore, because the pre-specified auditing strategies have low power to uncover failure modes, their tests frequently reject the auditor's null.
While this is mathematically correct since the AI system has passed these specific audits, it risks being misinterpreted as indicating robustness of the AI system.
Finally, the \texttt{Oracle} auditor outperforms all the existing methods since it knows the true failure mode, showing that rejection can be as fast as 15--25 observations.
This shows that if one has an even smarter auditing strategy, such as through prior knowledge, one may be able to complete the statistical test with even smaller test suites.

When the model's null is true, every method correctly rejects $H_0^{\aud}$ at high rates, certifying that the audit has passed.
This confirms that Type~I error for $H_0^{\model}$ is controlled: no method falsely reports a failure mode when none exists.
We highlight that substantially more observations are needed by the testing procedures that use more adaptive auditors as they are better at finding poor-performing subgroups.
This is desirable, as it means the auditing strategy was indeed quite good and a lot more data was needed to confidently reject the auditor's null.
% PROPOSED EDIT (streamline null-true paragraph):
% When the model's null is true, all methods correctly reject $H_0^{\aud}$ at high rates, confirming Type~I error control for $H_0^{\model}$.
% Notably, stronger auditors require more observations before the audit passes---a desirable property indicating thorough exploration.

\subsection{Experiment 2: Auditing an LLM pipeline for clinical notes}

\begin{figure}
    \centering
    \vspace{-0.3cm}
    \includegraphics[width=0.4\linewidth]{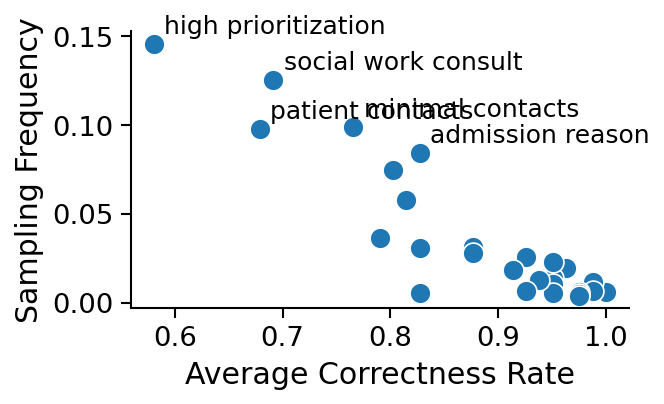}
    \includegraphics[width=0.4\linewidth]{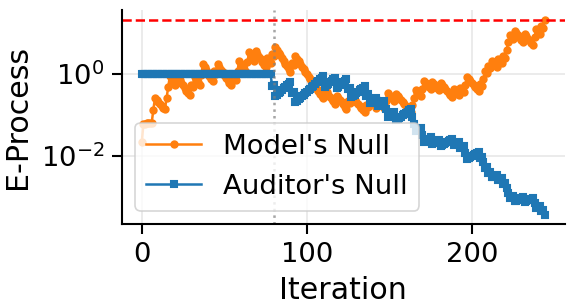}
    \includegraphics[width=0.4\linewidth]{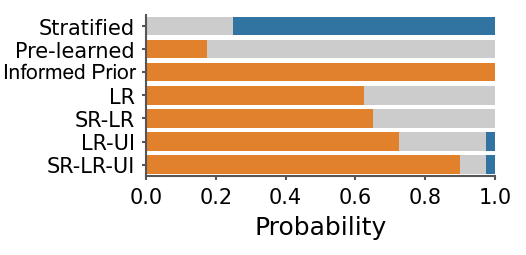}
    \includegraphics[width=0.4\linewidth]{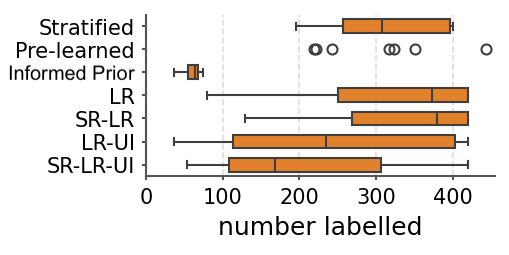}
    \vspace{-0.2cm}
    \caption{
        Auditing an LLM pipeline for extraction of Social Determinants of Health (SDoH) from clinical notes (Experiment 2).
        Top left: Sampling frequency vs.\ accuracy of different categories of SDoH clinical notes for \texttt{SR-LR-UI}, showing that the auditor learns to upsample low-accuracy categories.
        Top right: Example dual e-process with orange and blue curves representing the tests for the model's and auditor's null, respectively.
        The test detects that the AI system is not robust, as the orange curve exceeds the threshold of $1/\alpha$ (red dashed line).
        % Red dashed line is the rejection threshold of $1/\alpha$.
        Bottom left: Rates at which the AI system fails the robustness audit (orange) versus passage of the audit (blue); gray is inconclusive.
        Bottom right: Boxplots of the number of clinical notes annotated until robustness is rejected or the maximum sample size is reached.
        %\red{make smaller plots}
    }
    \label{fig:exp2}
\end{figure}

We next applied our framework to a real-world task: auditing an LLM-based pipeline originally developed by \citet{Kothari2026-zo} that extracts social determinants of health (SDoH) from clinical notes across 28 categories, including housing stability, substance use, mental health, safety concerns, and patient support networks.
The original dataset includes human annotations across 81 clinical notes for all categories, yielding 2268 total comparisons.
The overall LLM-human agreement rate was 88\%, but category-level accuracy varied from 58\% to 100\%.
To evaluate the proposed auditing procedures on this dataset, we set the auditing threshold to $q = 0.83$.

Figure~\ref{fig:exp2} reveals similar qualitative patterns in real data, where \texttt{SR-LR-UI} achieves the highest power and requires the fewest samples for rejecting the model's null (except for \texttt{Informed Prior} which knows which SDoH categories perform poorly).
Visualizing which SDoH categories are being selected by the adaptive auditor, we also find that the adaptive auditor indeed learns from just query embeddings to concentrate on the low-accuracy categories (Figure~\ref{fig:exp2} top left), with sampling frequency strongly negatively correlated with the category's accuracy (Spearman $\rho = -0.84$).

\section{Discussion}
\label{sec:discussion}
To understand the statistical underpinnings of highly flexible and targeted evaluation of AI workflows, this work introduces a formal hypothesis testing framework with procedures that provide anytime-valid guarantees under arbitrary adaptive sampling and optional stopping.
The key insight is to represent robustness auditing of AI systems from two dueling perspectives: the model's null that asserts there are no failure modes, and the auditor's null that asserts they will find one.
Then using the testing-by-betting framework, we introduce e-process-based procedures that maintain finite-sample Type~I error control regardless of how adaptively the test cases are selected.
Code for reproducing all results is available at  \url{https://github.com/jjfenglab/safe-adaptive-testing}.

Future directions of this work include measuring and optimizing auditor quality; while this work provides a statistical foundation for any black-box auditor, the reliability of its outputs depends on the auditor's strength.
% We leave the optimization of auditing strategies to future work.
Also, our empirical evaluations used an automated, albeit unpredictable, auditor to comprehensively test the statistical procedures across a variety of settings.
Future work should study the suitability of this framework with not only more sophisticated automatic auditors but also human-in-the-loop systems where humans and AI agents collaborate to find failure modes (see e.g., ProbeLLM \citep{Huang2026-rl} or Domino \citep{Eyuboglu2023-yn}).
Ultimately, by providing a general statistical framework for common auditing practices of modern AI systems, we hope to empower practitioners to confidently assess AI robustness without sacrificing statistical validity.

\section*{Acknowledgments}
The authors thank Jialin Ouyang, Hanyang Li, Avni Kothari, and Roman Kouznetsov for helpful discussions on this project.
This work was supported by the Weill Cancer Hub
West and the Weill Family Foundation.

\bibliographystyle{plainnat}
\bibliography{paperpile}

\appendix
% \section{Appendix}

\renewcommand\thefigure{\thesection.\arabic{figure}}    
\setcounter{figure}{0}

\section{Proofs}

\begin{proof}[Proof of Theorem~\ref{thm:duality}]
    Recall the model's hypotheses are
    \begin{align*}
    \begin{split}
    H_0^{\model}: & \quad \forall S \in \mathcal{S}_{\epsilon}, \quad \Acc(S) \geq q,\\
    H_1^{\model}: & \quad \exists S \in \mathcal{S}_{\epsilon} \text{ such that } \Acc(S) < q.
    \end{split}
    \end{align*}
    and the auditor's hypotheses can be formulated as:
    \begin{align*}
    % \begin{split}
    %     H_0^{\aud, [m]}: &  \ \exists \tau \in [m] \text{ such that } \forall t \geq \tau, \ \Acc(\Stp) < q,\\
    %     H_1^{\aud,[m]}: & \not \exists \tau \in [m] \text{ such that } \forall t \geq \tau, \ \Acc(\Stp) < q.
    % \end{split}
    \begin{split}
        H_0^{\aud, [m]}: &  \ \exists \tau \in [m] \text{ such that } \forall t \geq \tau, \ \Acc(\Stp) < q,\\
        H_1^{\aud,[m]}: & \forall \tau \in [m], \exists t \geq \tau \text{ such that }  \ \Acc(\Stp) \geq q.
    \end{split}
    \end{align*}
    
    As $m \rightarrow \infty$, denote
    \begin{align}
    H_0^{\aud, \mathbb{N}}: = 
    \lim_{m \rightarrow \infty} H_0^{\aud, [m]}: &  \quad \exists \tau \in \mathbb{N} \text{ such that } \forall t \ge \tau, \Acc(S_t^{\pi}) < q \ \\
    H_1^{\aud, \mathbb{N}}: = 
    \lim_{m \rightarrow \infty} H_1^{\aud, [m]}: & \quad \forall \tau \in \mathbb{N}, \ \exists t \geq \tau \text{ such that } \Acc(S_t^{\pi}) \geq q. 
    %\lim_{m \rightarrow \infty} H_1^{\aud, [m]}: & \quad \not \exists \tau \in \mathbb{N} \text{ such that } \forall t \geq \tau, \ \Acc(\Stp) < q.
    \end{align}
    The limits are well defined since both $H_0^{\aud, [m]}$ and $H_1^{\aud, [m]}$ are monotone in $m$.

    % \begin{figure}[ht]
    %     \centering
    %     \includegraphics[width=0.5\linewidth]{images/Duality.jpg}
    %     \caption{Visualization of duality}
    %     \label{fig:dualilty}
    % \end{figure}

    % \begin{theorem}[Duality of robustness and the auditor's hypotheses]
    % \label{thm:duality_append}
    % Under an asymptotically consistent and monoto auditing strategy $\pi$, the robustness null and audit hypotheses are complementary.
    % That is, the robustness null hypothesis holds iff the auditor's alternative hypothesis holds and vice versa:
    %     \begin{align}
    %         H_1^{\aud} \text{ true} &\iff H_0^{\model}\text{ true} \\
    %         H_0^{\aud} \text{ true}  &\iff H_1^{\model}\text{ true}.
    %     \end{align}
    % \end{theorem}

    Since $H_0^{\model}$ and $H_1^{\aud, \mathbb{N}}$ are the respective complements of $H_1^{\model}$ and $H_0^{\aud, \mathbb{N}}$, we only need to prove $H_1^{\model} \Leftrightarrow H_0^{\aud, \mathbb{N}}$. 
    
    Notice that if $H_0^{\aud, \mathbb{N}} $ is true, then there exists a large $T$ such that $\Acc(S_{T}^{\pi}) < q$ and $H_1^{\model}$ is true by taking $S^* = S_T^\pi$. On the other hand, if $H_1^{\model}$ is true, then by asymptotic consistency of the auditor's strategy, $\lim_{t \rightarrow \infty} \Acc(\Stp) < q $ and thus $H_0^{\aud, \mathbb{N} }$ is true. Therefore, $H_1^{\model}$ is true iff $H_0^{\aud, \mathbb{N}}$.
    % Denote events 
    % \begin{align*}
    %     %\Etp &= \left \{\Acc(S_t^\pi) < q \right\},    \\
    %     A &= \left\{ \lim_{t\rightarrow \infty} \Acc(S_t^\pi) < q \right\} = \cup_{\tau = 1}^{\infty} \cap_{t = \tau}^{\infty} \Etp,  \\
    %     B &= \left \{ \forall \ t, \ \Acc(S_t^\pi) \geq q \right\} = \cap_{t = 1}^{\infty} \Etp^c.
    % \end{align*}

    % \begin{enumerate}
    %     \item $H_0^{\model} \Rightarrow B$: always hold.
    %     \item $B \Rightarrow H_1^{\aud}$: $B = \cap_{t = 1}^{\infty} \Etp^c \subseteq \cap_{\tau = 1}^{\infty} \cup_{t = \tau}^{\infty} \Etp^c = H_1^{\aud}$.
    %     \item[3./6.] $H_0^{\aud} \Leftrightarrow A$: by definition.
    %     \item[4.] $A \Rightarrow H_1^{\model}$: $A$ implies that $\exists$ large $t$ s.t. $\Acc(S_t^\pi) < q$. Take $\Tilde{S} = S_t^{\pi}$.
    %     \item[5.] $H_1^{\model} \Rightarrow A$: guaranteed by Asymptotic Consistency.
    %     \item[7.] $H_1^{\aud} \Rightarrow B$: by monotonocity of $\pi$, $\Etp^c = \cup_{\tau = t}^\infty E_{\tau, \pi}^c \ \Rightarrow H_1^{\aud} = \cap_{\tau=1}^{\infty} \cup_{t = \tau}^{\infty} \Etp^c = \cap_{\tau = 1}^\infty \Etp^c = B $.
    %     \item[8.] $B \Rightarrow H_0^{\model}$: Asymptotic Consistency 
    % \end{enumerate}
\end{proof}

\begin{proof}[Proof of Theorem~\ref{thm:auditor_start}]
    For any $m' \in [m]$, let $S^{1, \pi} = \left\{S_1^{1, \pi}, S_2^{1, \pi}, \dots,  \right\}$ be any sequence of subsets which are identified by the auditor's sampling strategy $\pi$ and satisfy $H_1^{\aud, m'}$. 
    Suppose, to the contrary, that $\sup_{t < m'} \mathbb{E}(\phi(t) ; S_1^{1, \pi}, \dots, S_t^{1, \pi}) > \alpha$, i.e., the power can be above $\alpha$ at some $t < m' $.
    
    Next, construct $S^{0, \pi} = \left\{ S_1^{0, \pi}, S_2^{0, \pi}, \dots, \right\} $ such that $S^{0, \pi}$ satisfies $H_0^{\aud, m'} $ and  $S_t^{0, \pi} = S_t^{1, \pi}$ for $t < m' $. %This is possible because $H_0^{\aud} $ and $H_1^{\aud} $ are both tail events.
    However, this implies 
    \begin{align*}
        \sup_{t < m'} \mathbb{E}(\phi(t) ; S_1^{0, \pi}, \dots, S_t^{0, \pi}) = 
        \sup_{t < m'} \mathbb{E}(\phi(t) ; S_1^{1, \pi}, \dots, S_t^{1, \pi}) > \alpha,
    \end{align*}
    which contradicts $\sup_t \mathbb{E}\left (\phi(t)|H_0^{\aud, m'} \text{ is true} \right) \leq \alpha $. Therefore, by contradiction, $\sup_{t < m'} \mathbb{E}\left(\phi(t) | H_1^{\aud, m'} \text{ is true}\right) \leq \alpha$, which holds for any $m' \in [m]$.    
    
\end{proof}

\begin{proof}[Proof of Theorem~\ref{thm:fwer}]

If the model's null is true, then the auditor's null cannot be true.
If the auditor's null is true, then the model's null cannot be true.
Since only one of these null hypotheses can be true, then if both of the individual tests are controlled at level $\alpha$, the FWER is controlled at level $\alpha$.
\end{proof}

\section{Testing procedures}

\subsection{Pseudocode for dual testing procedure}

The pseudocode for dual testing is given in Algorithm~\ref{algo:dual_testing}.
The model's test $E_t^{\model}$ can be $E_t^{\SPRT, \model}$, $E_t^{\UI, \model}$, $E_t^{\SR\SPRT, \model}$ or $E_t^{\SR\UI, \model}$ and the auditor's test $E_t^{\aud}$ can be $E_t^{\SPRT, \aud}$ or $E_t^{\UI, \aud}$. 
% In our experiments, the model's and auditor's tests are matched according to adaptivity so that the dual processes $\left(E_t^{\model}, E_t^{\aud}\right)$ are both nonadaptive or both adaptive, though this is not neceesary in practice.
% Thus, the four candidates of $\left(E_t^{\model}, E_t^{\aud}\right)$ considered are $\left(E_t^{\SPRT, \model}, E_{t}^{\SPRT, \aud}\right)$, $\left(E_t^{\SR\SPRT, \model}, E_{t}^{\SPRT, \aud}\right)$, $\left(E_t^{\UI, \model}, E_{t}^{\UI, \aud}\right)$ and $\left(E_t^{\SR\UI, \model}, E_{t}^{\UI, \aud}\right)$.

\begin{algorithm}[H]
\caption{Dual testing}
\label{algo:dual_testing}
\begin{algorithmic}[1]
\State Initialize $E_0^{\model} = E_{0}^{\aud} = 1$, an auditing strategy $\pi$.
% \For{$t = 1, \ldots, n_{\text{init}}$}  \Comment{Burn-in: warm-start $\pi$, no testing}
%     \State Sample $X_t$ uniformly at random from $\mathcal{S}$.
%     \State Obtain score $Y_t$ for $X_t$.
%     \State Update $\pi$ using $(X_t, Y_t)$.
% \EndFor
% \For{$t = n_{\text{init}} + 1, n_{\text{init}} + 2, \ldots$}
%     \State ... (rest of loop)
% \EndFor
\For{$t = 1, 2, \ldots$}
    \State Pick subgroup $\Stp$ based on the audit strategy $\pi$.
    \State Sample test observation $X_t$ from $\Stp$.
    \State Evaluate the AI system's response on $X_t$ and obtain the score $Y_t$.
    \State Update $E_t^{\model}$.
    \If{$E_t^{\model} > 1/\alpha$}
        \State \Return Reject $H_0^{\model}$: Failure mode is detected.
    \EndIf
    \If{$t\geq m$}
        \State Update $E_{t}^{\aud}$.
        \If{$E_t^{\aud} > 1/\alpha$}
        \State \Return Reject $H_0^{\aud, m}$: Audit is passed.
    \EndIf
    \EndIf
\EndFor
\end{algorithmic}
\end{algorithm}

\subsection{Auditing strategies}
\label{sec:active_learning}

With rigorous finite-sample inference being possible under active testing based on the theoretical results above, the practical value of this framework depends on having a good auditing strategy $\pi$ for discovering failure modes.

\textbf{Scenario (a): Unlabeled data.}
If one has unlabeled data from the target population, one could learn a poor-performing subgroup with online learning. 
For instance, one may (re)train an AI/ML model $\hat{\pi}$ at the start of each iteration $t$ on labeled data collected up to time $t-1$ to estimate the conditional expectation of the score $E[Y|X]$.
This can then be used to determine sampling.
An appropriate exploitation-exploration tradeoff should be made, to ensure asymptotic consistency of the auditor.
In our experiments, the auditor greedily defines the subset $\Stp$ as the observations $X$ predicted to be in the bottom $\epsilon$ percentile \textit{and} have predicted scores $\hat{\pi}(X)$ below the target threshold $q$.
If such an $\epsilon$-sized subgroup does not exist, then the auditor favors exploration at time $t$ instead by sampling among observations whose lower confidence bound of the predicted score falls within the bottom 50\%, following the UCB strategy in the bandit literature \citep{Auer2002-rl}.

%simply by training an online learning procedure to predict accuracy, based on the sequence of observations that are active labelled through this active testing procedure.

%If there exists at least a proportion $\epsilon$ of observations whose scores are below the threshold $q-\delta$, then one observation is sampled from this group to be tested. 
%Otherwise, randomly sample from the remaining one. 

%Specifically, let $S_t$ denote the cumulative loss up to time $t$ and $V(X_{t+1})$ the individual score at $t+1$. Then $S_{t+1} = (1-\lambda) S_t + \lambda V(X_{t+1})$. 

% For example, suppose the target is to test the robustness of an LLM model. One begins by randomly sampling a group LLM queries and revealing the labels  whether LLM performs satsificatorily $\{Y_{0, i}\}_{i=1}^n$. Then an ensemble of NNs are trained on the embeddings of the queries and the labels. 

% Implementation options that we can discuss:
% \begin{itemize}
%     \item online learning with embedding model, greedy selection. we use ensemble of NN for the learning model, following \citep{Lakshminarayanan2017-wn}.
% \end{itemize}

\textbf{Scenario (b): Data generation.}
In certain settings, one may not have a pre-existing unlabeled dataset $\{X_j\}$.
This can occur, for instance, when one is building a new AI pipeline, where no data has yet been collected from the target application.
To audit such an AI pipeline, the auditor must generate $X$ values to audit instead.
In the ideal setup, the auditor can (i) determine that its selected subgroup $\Stp$ has prevalence at least $\epsilon$ and (ii) generate observations from $\Stp$ with respect to the probability measure $\mu$ restricted to the subgroup.
While the former is plausible as long as the auditor has adequate prior knowledge, the latter is highly dependent on the auditor's generation capabilities in practice (both in the setting where $X$ is constructed manually or automatically using a foundation model).
In such cases, an accurate and more honest description of the audit should be with respect to the actual data generation mechanism, which is an approximation of the probability measure in the target application.
For instance, if a failure mode is detected under an approximate generator, it means that the auditor has found some subgroup that the AI system performs poorly for, though this subgroup may have smaller prevalence than the target threshold of $\epsilon$.
If the auditor using the generator failed to reject the model’s null hypothesis (i.e., there is no failure mode), it means that the auditor failed to find a failure mode for the generator’s weighted (or biased) view of the target population, not necessarily that there is no failure mode globally.
The interpretation of the auditor’s null should be adjusted similarly.

\subsection{Practical guide for hyperparameters}

\paragraph{Stopping rule.}
The test stops when the e-statistic exceeds $1/\alpha$, where $\alpha$ is the desired Type I error rate.
The user only needs to specify $\alpha$, where the most typical choice is 0.05.

\paragraph{Hyperparameters defining the hypothesis test.}
Prevalence threshold $\epsilon$ and performance threshold $q$ define what a failure mode is and what the null hypothesis is.
Budget $m$ determines how many observations the auditor can explore prior to the auditor’s null being tested, and defines the auditor’s null hypothesis.
All three should be chosen based on prior knowledge and the goal of the auditor.
For instance, one may choose $q$ and $\epsilon$ based on regulatory or model fairness requirements.
Budget $m$ may be chosen to guard against detection of a failure mode by a typical user testing out the LLM application (so set $m$ to the typical number of attempts) or by some function class (so set $m$ to the typical training sample size needed).
We think these hyperparameters can often be chosen based on prior knowledge and by working closely with domain experts.
Nevertheless, in the absence of any prior knowledge, one could test a range of values and check if results remain robust, which can be viewed as a sensitivity or stability analysis (such as from veridical data science \citep{Yu2020-us}).
For example, ablation studies in Section C of the Appendix show that the relative performance of different methods is robust to the choice of $m$.

\paragraph{Auditor hyperparameters.}
We recommend tuning auditor hyperparameters in an online approach, such as the Exponentially Weighted Averaging Forecaster (EWAF) algorithm  \citep{Cesa-Bianchi2006-tl}.

\section{Experiment 1: Detailed Setup}
\label{app:exp1_details}

We provide complete details for the semi-synthetic experiments described in Section~\ref{sec:experiments}.

\paragraph{Data and label simulation.}
We use the GPQA dataset~\citep{Rein2023-dv} which contains 448 graduate-level multiple-choice science questions spanning three domains: Chemistry (187 questions), Biology (142 questions), and Physics (119 questions). Each question is also categorized by education level: easy undergraduate, hard undergraduate, hard graduate, and post-graduate.

We simulate binary accuracy labels (correct/incorrect) by assigning each domain$\times$education level combination a ground-truth accuracy probability. As an illustration, Table~\ref{tab:difficulty_settings_1a} shows the accuracy probabilities in Scenario (a) discussed below for Chemistry (the poor-performing subgroup) across education levels; Biology and Physics maintain accuracy $\geq 0.85$ in all settings.
The null threshold is $q = 0.85$ throughout.
% More details of all simulation settings can be found at \url{https://anonymous.4open.science/r/safe-active-testing-4281}.

% \begin{table}[h]
% \centering
% \begin{tabular}{lcccc}
% \toprule
% Degradation & Chemistry Accuracy & Gap from $q$ & Model's Null \\
% \midrule
% Large & 0.25--0.40 & 0.45--0.60 & False \\
% Medium & 0.68--0.72 & 0.13--0.17 & False \\
% Small & 0.70--0.75 & 0.10--0.15 & False \\
% None  & 0.92--0.96 & $> 0$ (above $q$) & True \\
% \bottomrule
% \end{tabular}
% \caption{Simulated accuracy settings for Chemistry across magnitude of the degradation. Biology and Physics domains maintain accuracy $\geq 0.85$ in all settings.}
% \label{tab:difficulty_settings}
% \end{table}

\begin{table}[h]
\centering
\begin{tabular}{lcccc}
\toprule
Degradation & Chemistry Accuracy & Gap from $q$ & Model's Null \\
\midrule
Large & 0.25--0.40 & 0.45--0.60 & False \\
Medium & 0.50--0.75 & 0.35--0.10 & False \\
Small & 0.60--0.77 & 0.25--0.08 & False \\
None  & 0.92--0.96 & $> 0$ (above $q$) & True \\
\bottomrule
\end{tabular}
\caption{Simulated accuracy settings for Chemistry across magnitude of the degradation in Experiment 1 Scenario (a). Biology and Physics domains maintain accuracy $\geq 0.85$ in all settings.}
\label{tab:difficulty_settings_1a}
\end{table}

% \begin{table}[!h]
% \centering
% \begin{tabular}{lcccc}
% \toprule
% Degradation & Chemistry Accuracy & Gap from $q$ & Model's Null \\
% \midrule
% Large & 0.25--0.40 & 0.45--0.60 & False \\
% Medium & 0.70--0.78 & 0.15--0.07 & False \\
% Small & 0.72--0.80 & 0.13--0.05 & False \\
% None  & 0.92--0.96 & $> 0$ (above $q$) & True \\
% \bottomrule
% \end{tabular}
% \caption{Simulated accuracy settings for Chemistry across magnitude of the degradation in Experiment 1 Scenario (b) \red{Move the table below scenario b?}. Biology and Physics domains maintain accuracy $\geq 0.85$ in all settings.}
% \label{tab:difficulty_settings_1b}
% \end{table}

\paragraph{Testing parameters.} The maximum sample budget is 250 observations. We initialize with $n_{\text{init}} = 2$ randomly sampled observations before active selection begins. The \texttt{Pre-learned} sampler uses an additional 30 held-out observations for pre-training.

\paragraph{Adaptive sampling.}
The adaptive sampling trains a model of the accuracy to identify the subgroup most likely to be a failure mode.
The accuracy model uses the text embedding of the GPQA question
% and its metadata (domain and education level) 
as input, generated from the \texttt{sentence-transformers/allenai-specter} model \citep{Cohan2020-jp}.
Accuracy is modeled using an ensemble of neural networks with one hidden layer, tuned over the hyperparameters of hidden dimension $ \in \{5, 10\}$ and number of epochs $\in \{2, 4, 8, 16\}$.
At each iteration, the best hyperparameter was based on the one with the smallest exponentially weighted held-out loss, which allows the selection procedure to adapt to shifts in the optimal hyperparameter as the accuracy model improved over time.
We use $\epsilon = 0.05$ for $\epsilon$-greedy exploration.

For UI-based estimators, we construct adaptive alternative probabilities in \texttt{\UI} and \texttt{\SR\UI} using the EWAF algorithm.
That is, the predicted probability for testing the model's null, $\hgamm$, is an adaptively weighted average over a pre-specified grid of $B$ plausible probabilities $\mathcal{Q}_{1}^{\model} = \{q_b\in (0, q): b \in [B]\}$ under $H_1^{\model}$.
In particular, $\hgamm = \sum_{b = 1}^{B} u_{t-1, b} q_b$ with the weights $u_{t-1,b}$ defined as
\begin{align*}
    u_{t, b} = \frac{u_{t-1,b}\exp\left(\lambda \log \left( \frac{p(Y_t;q_b)}{p(Y_t;q)} \right)\right)}{\sum_{b'=1}^B u_{t-1,b'}\exp\left(\lambda \log \left( \frac{p(Y_t;q_b')}{p(Y_t;q)} \right)\right)}.
\end{align*}
The predicted probability for testing the auditor's null, $\hat{\gamma}_{t-1}^{\aud}$, is computed similarly.

\paragraph{Scenario (a): Unlabeled data.}
The auditor has access to all 448 GPQA questions without labels.
At each iteration, the auditor selects one question to label based on predicted accuracy (greedy selection of lowest predicted accuracy, with $\epsilon$-greedy exploration).
The selected question is then removed from the pool.

\paragraph{Scenario (b): Data generation.}
The auditor generates a question from each of the $3 \times 3$ categories corresponding to the levels of domain and education level.
(To control computational costs, we pre-generate a question bank using an LLM for each category, rather than querying an LLM at every iteration.)
The auditor then selects among the generated candidates using $\epsilon$-greedy exploration.

\paragraph{Evaluation.}
We run 100 simulation replicates for each simulation setting.
For each replicate, we record: (1) which null hypothesis was rejected ($H_0^{\model}$ or $H_0^{\aud}$); (2) the stopping time (number of observations until correct rejection or the maximum sample size is reached). We report rejection rates and boxplots of stopping times across trials.

\paragraph{Compute time.}
The procedure is quick to run and requires minimal infrastructure.
With only two CPU cores with 5GB memory, a single audit completed within 15 minutes on average.
Most of the time was spent updating the auditor with each new observation, as the auditor conducted hyperparameter tuning of the NN ensemble at each iteration.
Significant speed-ups are possible by parallelizing hyperparameter tuning.

\paragraph{Ablation study on pre-auditor's budget $m$.} We perform an ablation study to examine the effect of the budget $m$ prior to starting the auditor's test on the performance of the proposed procedures. 
Specifically, we evaluate $m \in \{10, 40, 70\}$ while holding all other experimental settings constant. 
As shown in Figure~\ref{fig:exp1_ablation_m10} for $m=10$ and Figure~\ref{fig:exp1_ablation_m70} for $m=70$, results are qualitatively similar to those with $m = 40$ in Figure~\ref{fig:exp1} in the main text.
Thus, the procedure is relatively insensitive to the choice of budget $m$.

% \red{power of $\SR\UI$ increases with $m$ in both scenarios. However, power of $\UI$ increases in Scenario (b) but decreases in Scenario (a). Change is small for $\UI$. ARL of rejecting $H_0^{\model}$ remains almost the same and that of rejecting $H_1^{\aud}$ increases, which is expected since auditor's test starts late. Don't feel like we need to describe this.}

\begin{figure}[!ht]
    \centering
    \begin{minipage}{.5\textwidth}
        \centering
        {\small \textit{(a) Auditor selecting unlabeled cases to annotate}} \\
        \begin{minipage}{.49\textwidth}
            \includegraphics[width=\textwidth]{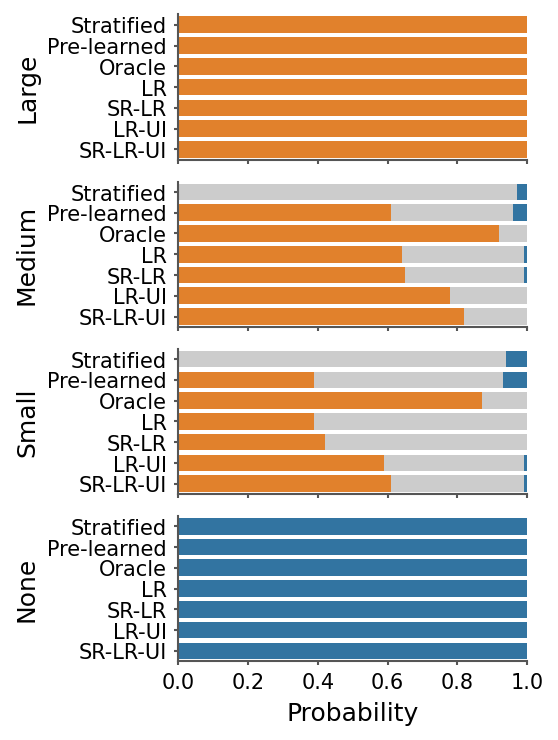}
        \end{minipage}
        \begin{minipage}{.49\textwidth}
            \includegraphics[width=\textwidth]{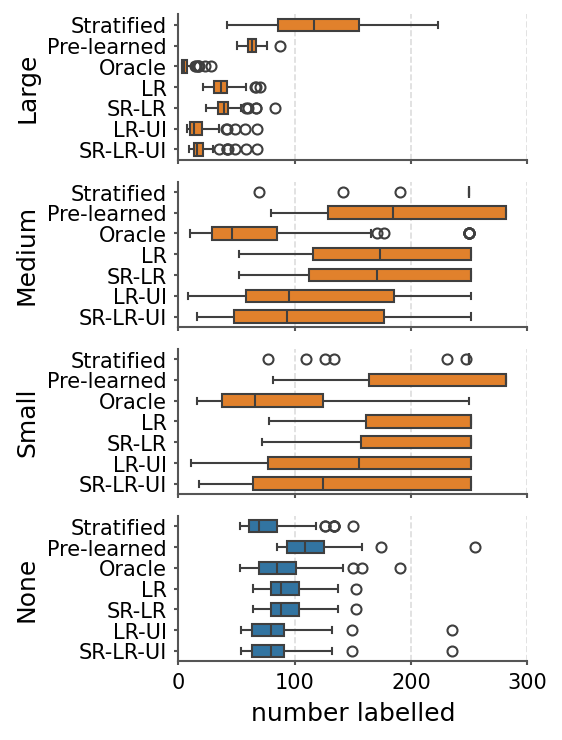}
        \end{minipage}
    \end{minipage}%`
    \begin{minipage}{.5\textwidth}
        \centering
        {\small \textit{(b) Auditor generates test cases to annotate}}\\
        \begin{minipage}{0.495\textwidth}
            \includegraphics[width=\textwidth]{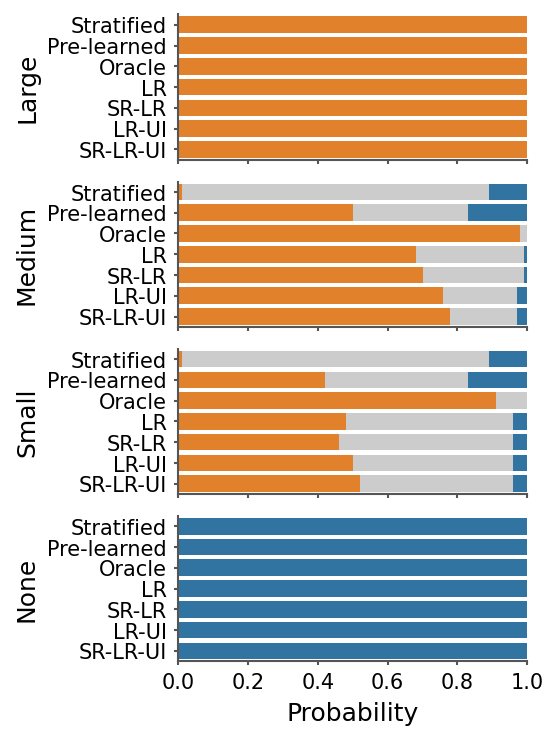}
        \end{minipage}%\hspace{8pt}%
        \begin{minipage}{0.495\textwidth}
            \includegraphics[width=\textwidth]{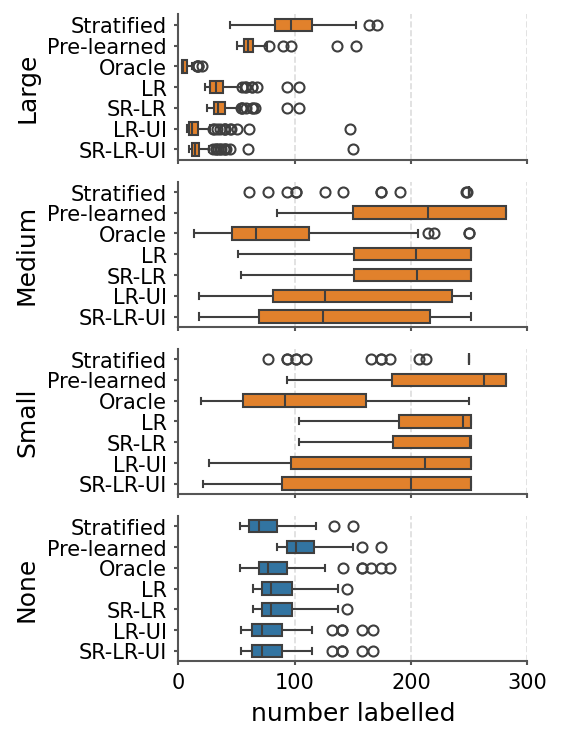 }
        \end{minipage}
    \end{minipage}
    \begin{minipage}{0.3\textheight}
        \centering
        \includegraphics[width = \textwidth]{images/plot_agg_legend.png}
    \end{minipage}
   \caption{Detecting failure modes with semi-synthetic data (Experiment~1) with $\mathbf{m=10}$. 
    The rows correspond to settings where the model has a failure mode with decreasing magnitude/prevalence of the degradation (\emph{Large}, \emph{Medium}, \emph{Small}), where the bottom corresponds to \emph{No failure mode}.
    Left column: rate of rejecting dueling null hypotheses, where orange indicates rate of detecting a failure mode, blue indicates passing the audit, and gray corresponds to neither (inconclusive).
    Right column: boxplots of the number of cases annotated during the audit to reach the correct conclusion or the maximum sample size.
    }
    \label{fig:exp1_ablation_m10}
\end{figure}

\begin{figure}[!ht]
    \centering
    \begin{minipage}{.5\textwidth}
        \centering
        {\small \textit{(a) Auditor selecting unlabeled cases to annotate}} \\
        \begin{minipage}{.49\textwidth}
            \includegraphics[width=\textwidth]{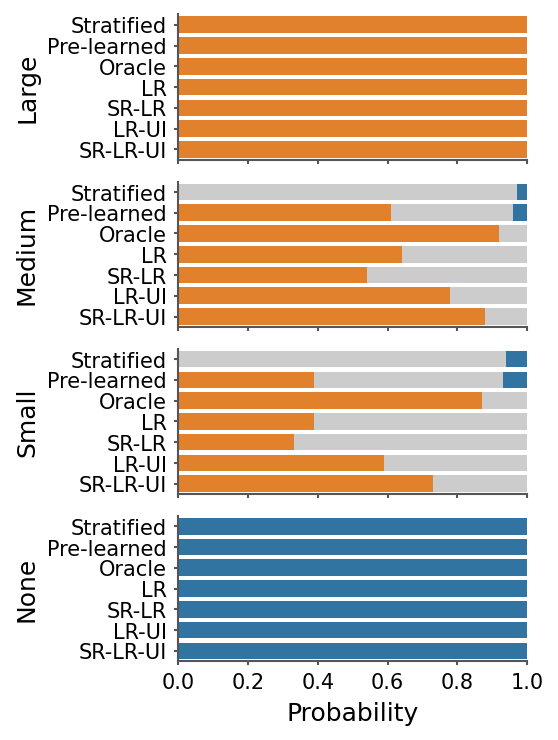}
        \end{minipage}
        \begin{minipage}{.49\textwidth}
            \includegraphics[width=\textwidth]{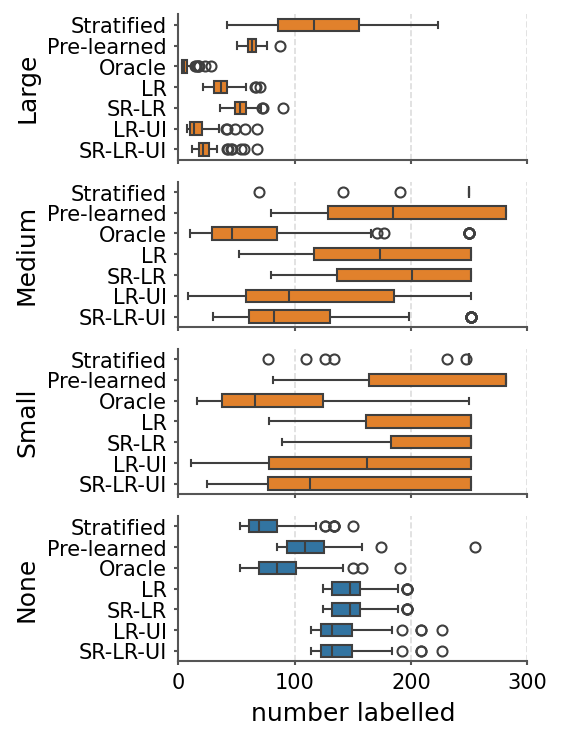}
        \end{minipage}
    \end{minipage}%`
    \begin{minipage}{.5\textwidth}
        \centering
        {\small \textit{(b) Auditor generates test cases to annotate}}\\
        \begin{minipage}{0.495\textwidth}
            \includegraphics[width=\textwidth]{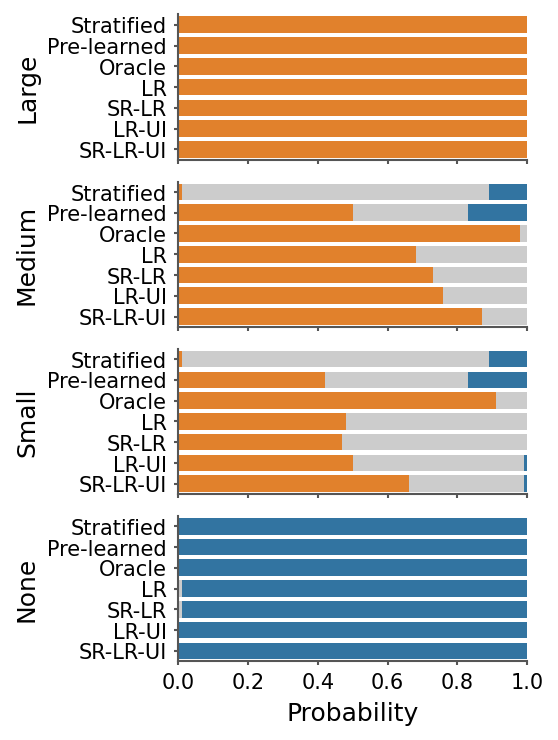}
        \end{minipage}%\hspace{8pt}%
        \begin{minipage}{0.495\textwidth}
            \includegraphics[width=\textwidth]{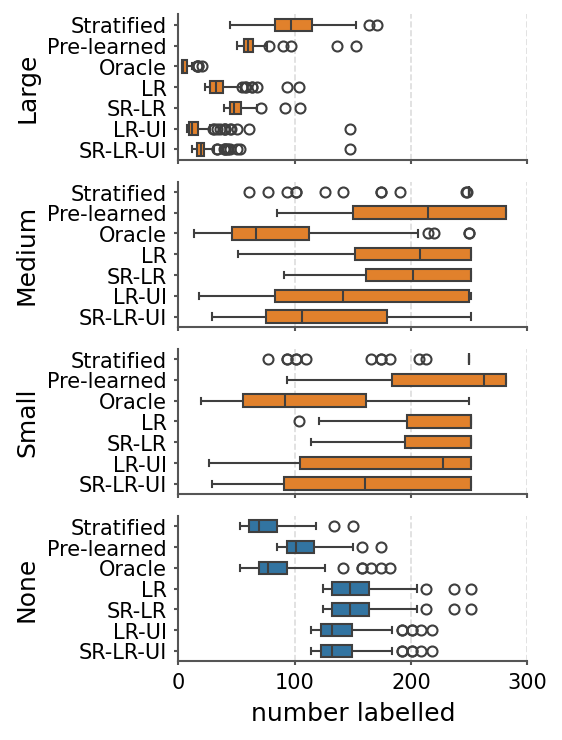 }
        \end{minipage}
    \end{minipage}
    \begin{minipage}{0.3\textheight}
        \centering
        \includegraphics[width = \textwidth]{images/plot_agg_legend.png}
    \end{minipage}
    \vspace{-0.2cm}
    \caption{Detecting failure modes with semi-synthetic data (Experiment~1) with $\mathbf{m=70}$. 
    The rows correspond to settings where the model has a failure mode with decreasing magnitude/prevalence of the degradation (\emph{Large}, \emph{Medium}, \emph{Small}), where the bottom corresponds to \emph{No failure mode}.
    Left column: rate of rejecting dueling null hypotheses, where orange indicates rate of detecting a failure mode, blue indicates passing the audit, and gray corresponds to neither (inconclusive).
    Right column: boxplots of the number of cases annotated during the audit to reach the correct conclusion or the maximum sample size.
    }
    \label{fig:exp1_ablation_m70}
\end{figure}

\section{Experiment 2: Detailed Setup}
\label{app:exp2_details}

We provide complete details for the clinical NLP experiment described in Section~\ref{sec:experiments}.

\paragraph{Data and task.}
We evaluate an LLM-based pipeline for extracting social determinants of health (SDoH) from clinical notes at an urban medical center.
The pipeline processes discharge summaries and social work notes to extract structured information across 28 SDoH categories, including housing stability, substance use patterns, mental health needs, safety concerns, patient contacts, and care coordination details.
The dataset consists of 81 clinical notes with human expert annotations for each of the 28 categories, yielding 2,268 total note$\times$category comparisons.
Each comparison is scored as correct (1) if the LLM extraction matches the human annotation, and incorrect (0) otherwise.
Table~\ref{tab:sdoh_categories} provides the complete list of SDoH categories with their definitions as specified in the LLM prompts.

\paragraph{Category-level performance.}
The overall LLM-human agreement rate is 88\%, but performance varies substantially across categories.
The worst-performing categories include: high prioritization (58\%), patient contacts (68\%), social work consult (69\%), minimal contacts (77\%), and mental health (79\%).
The best-performing categories achieve 95--100\% agreement, including substance use date of last use (100\%), immigration interventions (99\%), and several substance use subcategories.
With the null threshold set to $q = 0.83$, the performance on 11 of the 28 categories falls below $q$.

\paragraph{Adaptive learning.}
For each note$\times$category pair, we construct input features by concatenating (1) the category-specific prompt definition from Table~\ref{tab:sdoh_categories} and (2) the clinical note text.
We compute embeddings using the \texttt{sentence-transformers/allenai-specter} model \citep{Cohan2020-jp}.
The accuracy model is the same neural network ensemble as that in Experiment~1.

\paragraph{Testing parameters.}
We set $m = 40$ (audit budget), a maximum sample size of 400, and $n_{\text{init}} = 20$ initial random samples.
We report rejection rates and stopping times from 40 simulation replicates.

\paragraph{Additional results.}
We also carry out the experiment with a null threshold $q = 0.85$.
As shown in Figure~\ref{fig:exp2_q085}, the comparison is similar to the case of $q=0.83$ in the main text, with $\SR\UI$ performing the best.

\begin{figure}
    \centering
    \vspace{-0.3cm}
    \includegraphics[width=0.4\linewidth]{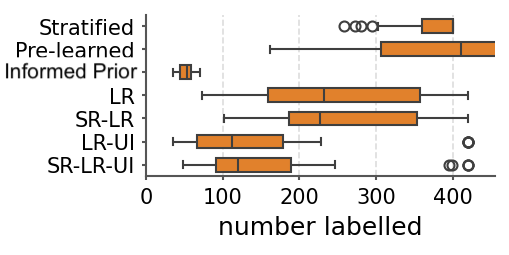}
    \includegraphics[width=0.4\linewidth]{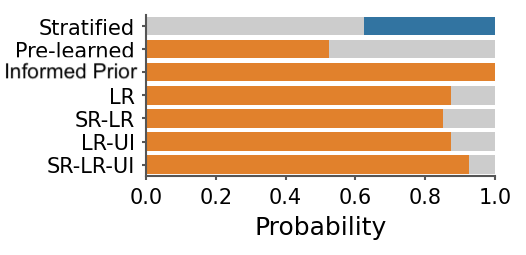}
    \vspace{-0.2cm}
    \caption{
        Auditing an LLM pipeline for extraction of Social Determinants of Health (SDoH) from clinical notes (Experiment 2) with threshold $q=0.85$.
        Left: Rates at which the AI system fails the robustness audit (orange) versus passage of the audit (blue); gray is inconclusive.
        Right: Boxplots of the number of clinical notes annotated until robustness is rejected or the maximum sample size is reached.
    }
    % \vspace{-0.5cm}
    \label{fig:exp2_q085}
\end{figure}

\begin{table}
\centering
\footnotesize
\begin{tabular}{p{3.5cm}p{11cm}}
\toprule
\textbf{Category} & \textbf{Definition} \\
\midrule
Reason for Admission & Identify the reason for the admission to the hospital. \\
\addlinespace
Reason for SW Consult & For social worker consult note types only, identify the reason for the social work consult. \\
\addlinespace
Mental Health & Identify any mention of mental health needs: depression, anxiety, suicidal thoughts, or psychiatric symptoms impacting functioning. \\
\addlinespace
Mental Health Interventions & Extract actions/interventions for mental health needs, categorized as previous, current, or planned actions. \\
\addlinespace
Safety & Identify any mention of safety needs: abuse, domestic violence, assault, interpersonal violence, recurrent falls, or self-neglect/caregiver neglect. \\
\addlinespace
Safety Interventions & Extract actions/interventions for safety needs, categorized as previous, current, or planned actions. \\
\addlinespace
Housing & Identify any mention of housing needs: homelessness, housing instability, unsafe living conditions, or inability to afford housing. \\
\addlinespace
Housing Interventions & Extract actions/interventions for housing needs. Look for: home health, shelters, navigation centers, Medical Respite, SNF referrals, supportive housing. \\
\addlinespace
Substance Use & Identify any mention of substance use needs: illegal drugs, prescription drug misuse, or problematic controlled substance use. \\
\addlinespace
Substance Use Interventions & Extract actions/interventions for substance use needs. Look for: detox referrals, residential/outpatient SUD programs, Addiction Care Team. \\
\addlinespace
Substance Use - OUD & Opioid Use Disorder: Identify if the patient has problematic opioid use (including fentanyl or heroin). \\
\addlinespace
Substance Use - StUD & Stimulant Use Disorder: Identify if the patient has problematic stimulant use (amphetamines, cocaine, methamphetamine). \\
\addlinespace
Substance Use - PSUD & Polysubstance Use Disorder: Identify if the note explicitly mentions PSUD. \\
\addlinespace
Substance Use - MOUD & Medication for OUD: Identify if the patient is on methadone or buprenorphine. \\
\addlinespace
Substance Use - Date of Last Use & If the patient has substance use history, identify the date of last use. \\
\addlinespace
Food Insecurity & Identify any mention of food insecurity: inability to access adequate food due to financial constraints. \\
\addlinespace
Food Insecurity Interventions & Extract actions/interventions for food insecurity. Look for: Meals on Wheels, Great Plates, Project Open Hand. \\
\addlinespace
Immigration & Identify any mention of immigration needs: documentation concerns, deportation fears, language barriers. \\
\addlinespace
Immigration Interventions & Extract actions/interventions for immigration needs. \\
\addlinespace
Tobacco Use & Identify any mention of tobacco use: cigarettes, vaping, chewing tobacco, or other nicotine products. \\
\addlinespace
Tobacco Use Interventions & Extract actions/interventions for tobacco use needs. \\
\addlinespace
Alcohol Use & Identify any mention of alcohol use needs: excessive drinking, binge drinking, or alcohol-related problems. \\
\addlinespace
Alcohol Use Interventions & Extract actions/interventions for alcohol use needs. \\
\addlinespace
DME: Mobility Devices & Identify durable medical equipment the patient uses: wheelchairs, walkers, hospital beds, oxygen, CPAP/BiPAP, etc. \\
\addlinespace
Patient Contacts & Identify references to patient's contacts: friends, family, case workers. Include name, phone number, and relation. \\
\addlinespace
Lack of Social Support & Identify missing/minimal contacts or social isolation: widowed, lives alone, no family involved, no emergency contact. \\
\addlinespace
Patient Independence & Identify inability to perform ADLs (bathing, dressing, toileting, eating) or IADLs (medications, transportation, cooking, finances). \\
\addlinespace
High Priority & Identify if patient is high priority: homeless and discharge ready, Jane/John Doe, abuse/neglect/safety, need family located, death/dying, major life-changing event. \\
\addlinespace
Outpatient Therapies & Identify home health services or SNF: PT/OT/SLP/RN/SW that come to patient's home, or IHSS caregivers. \\
\bottomrule
\end{tabular}
\caption{SDoH categories and their definitions as specified in the LLM extraction prompts. Definitions are abbreviated for space; full prompts include additional examples and classification guidance.}
\label{tab:sdoh_categories}
\end{table}

\section{More Results: CUB-200-2011 Image Data}

% \begin{figure}
%     \centering
%     % \includegraphics[width=\linewidth]{images/experiment1_subgroups.pdf}
%     \includegraphics[width=\linewidth]{images/experiment2_subgroups.pdf}
%     \caption{ARL Subgroup sizes that were proposed during active testing, for data settings that are increasingly difficult. When null is true, the auditor decides to explore more and opts to sample from larger subgroups. When the auditor can easily find the subgroup, it only samples from very small targeted subgroups.}
%     \label{fig:exp2_subgroups}
% \end{figure}

\paragraph{Data and task.}
We evaluate an image classifier on the CUB-200-2011 dataset~\citep{Wah2011-kc}, which contains 11,788 images of 200 bird species.
The model under audit is a ResNet-18 classifier.
The auditor's goal is to determine whether there exists a subgroup whose accuracy falls below a null threshold $q = 0.85$.
Binary accuracy labels are simulated by assigning each species a ground-truth accuracy probability, with white species receiving the highest accuracy, darker species moderate accuracy, and 18 black species the lowest.

\paragraph{Adaptive learning.}
We trained an accuracy model using the same NN ensemble architecture with embeddings extracted from a pre-trained ResNet-18 \citep{He2015-no} as input features.

\paragraph{Testing parameters.}
We set the budget $m = 40$ prior to starting the auditor's test, a maximum sample budget of 250 observations, and $n_{\text{init}} = 20$ initial random observations before active selection begins.
The \texttt{Pre-learned} sampler uses an additional 30 held-out observations for pre-training.
We run 40 independent trials per method.
Results are shown in Figure~\ref{fig:cub-birds}.

\begin{figure}
    \centering
    \includegraphics[width=0.45\linewidth]{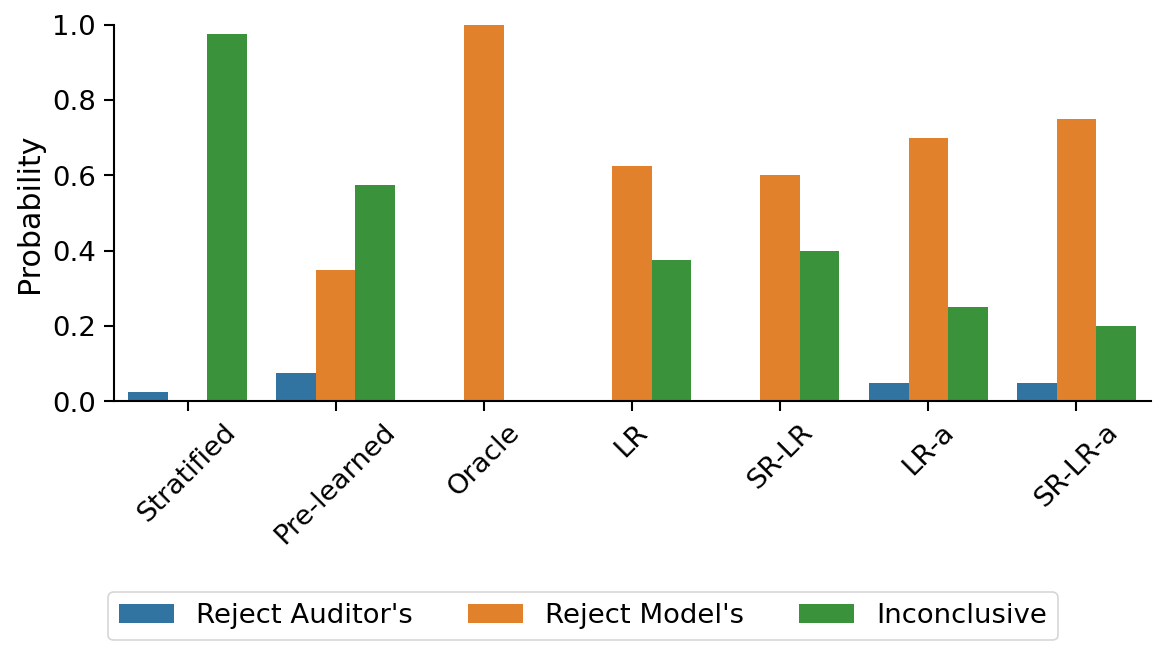}
    \includegraphics[width=0.45\linewidth]{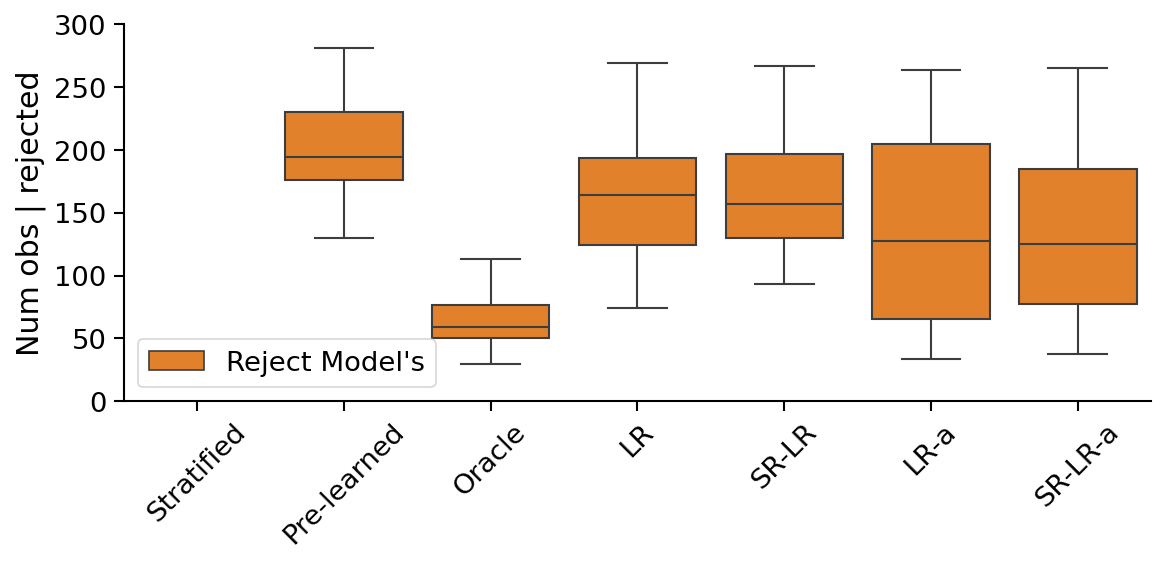}
    \caption{Auditing a pre-trained ResNet-18 classifier for CUB-200-2011 bird classification.
    Left: Rates for rejection of robustness (orange) and passage of the audit (blue).
    Right: Boxplots of the number of images annotated when robustness is rejected or the maximum budget is exhausted.}
    \label{fig:cub-birds}
\end{figure}

\newpage
\section*{NeurIPS Paper Checklist}

\begin{enumerate}

\item {\bf Claims}
    \item[] Question: Do the main claims made in the abstract and introduction accurately reflect the paper's contributions and scope?
    \item[] Answer: \answerYes{}
    \item[] Justification: The abstract and introduction state the three main contributions: formalizing AI auditing as dueling hypothesis tests, introducing dual e-processes with SAVI guarantees, and empirical demonstration of Type-I error control. These claims are directly supported by Theorems~\ref{thm:duality} and~\ref{thm:fwer} and the experimental results in Section~\ref{sec:experiments}.

\item {\bf Limitations}
    \item[] Question: Does the paper discuss the limitations of the work performed by the authors?
    \item[] Answer: \answerYes{}
    \item[] Justification: The Discussion (Section~\ref{sec:discussion}) acknowledges key limitations, including that the framework treats auditing strategies as a black box and that empirical evaluations focus on automated auditors rather than human-in-the-loop settings.

\item {\bf Theory assumptions and proofs}
    \item[] Question: For each theoretical result, does the paper provide the full set of assumptions and a complete (and correct) proof?
    \item[] Answer: \answerYes{}
    \item[] Justification: All theorems state their assumptions explicitly, and complete proofs are provided in the Appendix.

    \item {\bf Experimental result reproducibility}
    \item[] Question: Does the paper fully disclose all the information needed to reproduce the main experimental results of the paper to the extent that it affects the main claims and/or conclusions of the paper (regardless of whether the code and data are provided or not)?
    \item[] Answer: \answerYes{}
    \item[] Justification: All experimental results can be reproduced using the code provided at the GitHub repository linked in the Discussion, which includes scripts for all experiments described in Section~\ref{sec:experiments}.

\item {\bf Open access to data and code}
    \item[] Question: Does the paper provide open access to the data and code, with sufficient instructions to faithfully reproduce the main experimental results, as described in supplemental material?
    \item[] Answer: \answerYes{}
    \item[] Justification: Code for all experiments is available at the GitHub repository linked in the Discussion. Data for all experiments in the paper are publicly available except for Experiment 2.
    Experiment 2 uses real-world clinical notes that cannot be shared publicly due to HIPAA restrictions, but may be obtained from the authors of \citet{Kothari2026-zo} through an appropriate data sharing request.

\item {\bf Experimental setting/details}
    \item[] Question: Does the paper specify all the training and test details (e.g., data splits, hyperparameters, how they were chosen, type of optimizer) necessary to understand the results?
    \item[] Answer: \answerYes{}
    \item[] Justification: Section~\ref{sec:experiments} specifies all experimental settings; full implementation details are provided in the Appendix.

\item {\bf Experiment statistical significance}
    \item[] Question: Does the paper report error bars suitably and correctly defined or other appropriate information about the statistical significance of the experiments?
    \item[] Answer: \answerYes{}
    \item[] Justification: The paper introduces methods for establishing statistical significance. In addition, all experiments were comprehensively evaluated over multiple replicates to quantify statistical uncertainty, which are summarized in boxplots in Figure~\ref{fig:exp1} and Figure~\ref{fig:exp2}; additional results are also provided in the Appendix.

\item {\bf Experiments compute resources}
    \item[] Question: For each experiment, does the paper provide sufficient information on the computer resources (type of compute workers, memory, time of execution) needed to reproduce the experiments?
    \item[] Answer: \answerYes{}
    \item[] Justification: Details on compute resources used for all experiments are provided in the Appendix.
    
\item {\bf Code of ethics}
    \item[] Question: Does the research conducted in the paper conform, in every respect, with the NeurIPS Code of Ethics \url{https://neurips.cc/public/EthicsGuidelines}?
    \item[] Answer: \answerYes{}
    \item[] Justification: We have reviewed the NeurIPS Code of Ethics and confirm that this research conforms with it in every respect.

\item {\bf Broader impacts}
    \item[] Question: Does the paper discuss both potential positive societal impacts and negative societal impacts of the work performed?
    \item[] Answer: \answerYes{}
    \item[] Justification: The Discussion highlights how the method can empower practitioners to assess AI robustness in a statistically rigorous manner.
    The paper also warns against using weak auditing strategies, as they can be misinterpreted as certifying global robustness.
    
\item {\bf Safeguards}
    \item[] Question: Does the paper describe safeguards that have been put in place for responsible release of data or models that have a high risk for misuse (e.g., pre-trained language models, image generators, or scraped datasets)?
    \item[] Answer: \answerNA{}
    \item[] Justification: The paper does not release any models or datasets that pose a high risk for misuse.

\item {\bf Licenses for existing assets}
    \item[] Question: Are the creators or original owners of assets (e.g., code, data, models), used in the paper, properly credited and are the license and terms of use explicitly mentioned and properly respected?
    \item[] Answer: \red{\answerYes{}}
    \item[] Justification: All owners of data assets have been cited. Usage of the datasets abided by their terms. %\red{The GPQA dataset (CC BY 4.0) is properly cited; the clinical notes dataset from Experiment 2 is a restricted hospital dataset cited as \citet{Kothari2026-zo} and is not publicly released. (Mention the license in the appendix)}

\item {\bf New assets}
    \item[] Question: Are new assets introduced in the paper well documented and is the documentation provided alongside the assets?
    \item[] Answer: \answerYes{}
    \item[] Justification: The code released at the GitHub repository linked in the Discussion is documented with instructions for reproducing all experiments.

\item {\bf Crowdsourcing and research with human subjects}
    \item[] Question: For crowdsourcing experiments and research with human subjects, does the paper include the full text of instructions given to participants and screenshots, if applicable, as well as details about compensation (if any)?
    \item[] Answer: \answerNA{}
    \item[] Justification: The paper analyzes clinical notes in Experiment 2, which constitutes research with human subjects.

\item {\bf Institutional review board (IRB) approvals or equivalent for research with human subjects}
    \item[] Question: Does the paper describe potential risks incurred by study participants, whether such risks were disclosed to the subjects, and whether Institutional Review Board (IRB) approvals (or an equivalent approval/review based on the requirements of your country or institution) were obtained?
    \item[] Answer: \answerYes{}
    \item[] Justification: Appropriate IRB approval was obtained for analyzing the clinical notes in Experiment 2.

\item {\bf Declaration of LLM usage}
    \item[] Question: Does the paper describe the usage of LLMs if it is an important, original, or non-standard component of the core methods in this research? Note that if the LLM is used only for writing, editing, or formatting purposes and does \emph{not} impact the core methodology, scientific rigor, or originality of the research, declaration is not required.
    %this research?
    \item[] Answer: \answerNA{}
    \item[] Justification: LLMs are not the core component of the proposed methodology. They are only used to generate samples for auditing in Experiment 1b, which has been described in the Appendix.

\end{enumerate}
\end{document}